\documentclass[twoside]{article}
\def\preprint{1}
\if\preprint1
\usepackage[preprint]{aistats2026}
\else
  \usepackage{aistats2026}
\fi
\usepackage[utf8]{inputenc} 
\usepackage[T1]{fontenc}    
\usepackage{hyperref}       
\usepackage{url}            
\usepackage{booktabs}       
\usepackage{amsfonts}       
\usepackage{nicefrac}       
\usepackage{microtype}      
\usepackage{xcolor}         
\usepackage{float}

\usepackage{bm}
\usepackage{amssymb,amsfonts,amsmath,amssymb}
\usepackage[utf8]{inputenc}
\usepackage{dsfont}
\usepackage{tikz}
\usepackage{enumitem}
\usepackage{pgfplots}
\usepackage{graphicx}
\usepackage{amsthm}
\usepackage{mathtools}
\usepackage{comment}

\newtheorem{theorem}{Theorem}
\newtheorem{lemma}{Lemma}     
\newtheorem{corollary}{Corollary}
\newtheorem{proposition}{Proposition}

\theoremstyle{definition}
\newtheorem{assumption}{Assumption}

\newtheorem{remark}{Remark}

\newcommand{\N}{\mathbb{N}}
\newcommand{\R}{\mathbb{R}}

\newcommand{\E}{\mathbb{E}}
\renewcommand{\S}{\mathbb{S}}
\newcommand{\Prob}{\mathbb{P}}
\newcommand{\Q}{\mathbb{Q}}
\newcommand{\C}{\mathbb{C}}

\newcommand{\mb}{\mathbf}

\newcommand{\cS}{\mathcal{S}}
\newcommand{\cU}{\mathcal{U}}

\newcommand{\cX}{\mathcal{X}}
\newcommand{\cZ}{\mathcal{Z}}
\newcommand{\indic}{\mathds{1}}

\newcommand{\sign}{\mathrm{sign}}

\begin{document}

\twocolumn[

\aistatstitle{A Neural Network Algorithm for KL Divergence Estimation with Quantitative Error Bounds}

\aistatsauthor{ Mikil Foss \And Andrew Lamperski}

\aistatsaddress{ University of Minnesota \And  University of Minnesota} ]

\author{%
Mikil Foss \\
  Mathematics\\
  University of Minnesota\\
  Minneapolis, MN 55455 \\
  \texttt{foss0379@umn.edu} 
   \And
  Andrew Lamperski\\
  Electrical and Computer Engineering\\
  University of Minnesota\\
  Minneapolis, MN 55455\\
  \texttt{alampers@umn.edu}
}

\begin{abstract}
Estimating the Kullback-Leibler (KL) divergence between random variables is a fundamental problem in statistical analysis. For continuous random variables, traditional information-theoretic estimators scale poorly with dimension and/or sample size.  To mitigate this challenge, a variety of methods have been proposed to estimate KL divergences and related quantities, such as mutual information, using neural networks. The existing theoretical analyses show that neural network parameters achieving low error exist.  However, since they rely on non-constructive neural network approximation theorems, they do not guarantee that the existing algorithms actually achieve low error. In this paper, we propose a KL divergence estimation algorithm using a shallow neural network with randomized hidden weights and biases (i.e. a random feature method). We show that with high probability, the algorithm achieves a KL divergence estimation error of $O(m^{-1/2}+T^{-1/3})$, where $m$ is the number of neurons and $T$ is both the number of steps of the algorithm and the number of samples. 
\end{abstract}

\section{Introduction}

The Kullback-Liebler (KL) divergence is a common measure of differences between random variables. KL divergence and related information theoretic measured are commonly estimated for applications such as econometrics \cite{golan2008information}, neuroscience \cite{timme2018tutorial}, and ecology \cite{ulanowicz2001information}. While methods to estimate KL divergence using neural networks are well-known, \cite{belghazi2018mutual,nowozin2016f}, the estimation error of the existing algorithms is not quantified. The paper presents an algorithm using random feature neural networks with ReLU activations and gives quantitative error guarantees for its performance.

\paragraph*{Related Work}

KL divergence estimation has a long history which is reviewed in \cite{verdu2019empirical}. For continuous random variables, common approaches are based on quantization and density  estimation. 
Motivated by limitations in the scaling of these methods with respect to dimension and/or sample size, optimization-based methods emerged \cite{nguyen2010estimating,nowozin2016f,belghazi2018mutual,song2020Understanding}. These methods utilize variational characterizations of the KL divergence (and more general divergence measures) to reduce the estimation problem to functional optimization problems. 

The algorithm in this paper is based on the Mutual Information Neural Estimation (MINE) method from \cite{belghazi2018mutual}. The MINE method uses neural networks to estimate KL divergence from data, which gives an estimate of mutual information as a special case.
The work in \cite{belghazi2018mutual,sreekumar2022neural} quantifies how the error in the estimate converges to $0$, provided that the optimization problem can be solved. 
However, since the optimization problem from \cite{belghazi2018mutual} is non-convex, there is no guarantee that the gradient-based algorithm proposed in \cite{belghazi2018mutual} actually solves the problem.  

Other papers related to MINE methods include \cite{hu2024infonet,mirkarimi2022deep,mirkarimi2023benchmarking,mirkarimi2022neural,tsur2023neural,tsur2024rate,tsur2025neural,lin2019data,choi2022combating,chan2019neural,molavipour2020conditional,molavipour2021neural,molavipour2021neural-conditional}.


\paragraph*{Contribution.}
The main contribution is the design and analysis of a MINE-type algorithm for KL divergence estimation using a shallow random feature ReLU network. We show that with high probability, the algorithm achieves a KL divergence estimation error of $O(m^{-1/2}+T^{-1/3})$, where $m$ is the number of neurons and $T$ is both the number of steps of the algorithm and the number of samples. 

As a secondary contribution, in order to prove the error bounds
, we extend approximation results from \cite{lamperski2024approximation,lamperski2024functiongradientapproximationrandom}, which bound the worst-case error for function approximation with random feature ReLU networks. In particular, we show how to eliminate the  need for  affine features. See Subsection~\ref{ss:approximation} for more discussion on related approximation results.


\section{A KL Divergence Estimator with Guaranteed Error Bounds}

\paragraph{\bf Notation:}
$\R$ is the set of real numbers, $\C$ is the set of complex numbers, and $\N$ is the set of non-negative integers. 
For a vector, $v$, its $p$-norm is denoted by $\|v\|_p$ for $p\in [1,\infty]$. If $f:\mathcal{D}\to \C$ its $L^p$-norm is denoted by $\|f\|_{L^p(\mathcal{D})}$ for $p\in [1,\infty]$.
If $\Theta$ is a convex set, $\Pi_\Theta$ is the projection onto $\Theta.$ If $\cS$ is a set, $\partial \cS$ denotes is boundary and $\mathrm{int}(\cS)$ denotes its interior.   If $M$ is a matrix or vector, then $M^\top$ is its transpose.  Random variables are denoted as bold symbols. $\E[\bm{x}]$ denotes the expected value of $\bm{x}$. 

\subsection{Background: Mutual Information and KL Divergences}

\paragraph{\bf Kullback-Liebler Divergence.}
Let $\Prob$ and $\Q$ be probability measures over a space $\Omega$, such that $\Q$ is absolutely continuous with respect to $\Prob$.
If $\bm{x}$ is distributed according to $\Prob$ and $\bm{y}$ is distributed according to $\Q$, then the Kullback-Liebler (KL) divergence is given by 
$$
D_{KL}(\Prob\| \Q) = \E\left[\log\left(\frac{d\Prob}{d\Q}(\bm{x})\right)\right]
$$

The Donsker-Varadhan variational characterization gives an expression for the KL divergence as an optimization over functions:
$$D_{KL} (\mathbb{P} \| \mathbb{Q}) = \sup_{T: \Omega \to \R} \left(\E[T(\bm{x})] - \log(\E[e^{T(\bm{y})}]))\right).$$
For any constant, $\xi$, an optimal solution is given by $T(x)=\log\left(\frac{d\Prob}{d\Q}(x) \right)+\xi$. 

\paragraph{\bf Mutual Information.}

Let $\bm{a}$ and $\bm{b}$ be random variables over spaces $\mathcal{A}$ and $\mathcal{B}$, respectively, such that $(\bm{a},\bm{b})$ has joint distribution $\Prob_{AB}$, $\bm{a}$ has distribution $\Prob_A$ and $\bm{b}$ has distribution $\Prob_B$. When the joint distribution, $\Prob_{AB}$ is absolutely continuous with respect to the product distribution, $\Prob_{A}\otimes\Prob_B$, the mutual information is defined by:
$$
I(\bm{a};\bm{b})=D_{KL}(\Prob_{AB}\| \Prob_{A}\otimes \Prob_{B}).
$$

In particular, if $\mathcal{A}\times \mathcal{B}=\Omega$ is a subset of $\R^n$ and  $\Prob_{AB}$ has a density with respect to the Lebesgue measure, denoted by $p_{AB}$, then setting $x=(a,b)$ gives:
\begin{subequations}
\label{eq:densities}
\begin{align}
p_{A}(a)&=\int_{\mathcal{B}}p_{AB}(a,b)db\\
p_{B}(b)&=\int_{\mathcal{A}}p_{AB}(a,b)da\\
\frac{d\Prob_{AB}}{d\left(\Prob_A\otimes \Prob_B \right)}(x)&=\frac{p_{AB}(a,b)}{p_{A}(a)p_B(b)}.
\end{align}
\end{subequations}

\paragraph{\bf MINE Methods.}
MINE stands for Mutual Information Neural Estimator \cite{belghazi2018mutual}. The idea behind MINE methods is to use a neural network, $\psi(x,\theta)$, with parameters $\theta$, to approximate $T(x)$ in the Donsker-Varadhan characterization. Namely, let $\Prob=\Prob_{AB}$ and $\Q=\Prob_{A}\otimes \Prob_B$, so that $\bm{x}$ corresponds to $(\bm{a},\bm{b})$ drawn according to their joint distribution, while $\bm{y}$ corresponds to $(\hat{\bm{a}},\hat{\bm{b}})$ where $\hat{\bm{a}}$ and $\hat{\bm{b}}$ are independent random variables drawn according to $\Prob_A$ and $\Prob_B$, respectively. Then, as long as there are neural network parameters, $\theta$, and a constant $\xi$ such that $\psi(x,\theta)\approx \log\left(\frac{d\Prob}{d\Q}(x) \right)+\xi$, we will have
\begin{align}
\nonumber
I(\bm{a};\bm{b})&=D_{KL}(\Prob\|\Q)\\
\label{eq:MINE}
&\approx \max_{\theta}\left( \E[\psi(\bm{x},\theta)] - \log(\E[e^{\psi(\bm{y},\theta)}]))\right).
\end{align}

When $\log\left(\frac{d\Prob}{d\Q} \right)$ is sufficiently smooth, classical approximation theorems, such as described in \cite{pinkus1999approximation}, guarantee that good neural network approximations exist. However, the current theory of MINE algorithms does not explain whether the algorithms used in practice actually find good approximations. The challenge arises from two issues: 1) If $\psi(x,\theta)$ is a deep neural network,  then the optimization problem from \eqref{eq:MINE} is non-convex. 2) The logarithm does not commute with differentiation:
\begin{align*}
 \nabla_{\theta}\log(\E[e^{\psi(\bm{y},\theta)}]))&=\frac{1}{\E[e^{\psi(\bm{y},\theta)}]}\E\left[
e^{\psi(\bm{y},\theta)}\nabla_{\theta} \psi(\bm{y},\theta)
\right] \\
&\ne \E\left[\nabla_{\theta} \log e^{\psi(\bm{y},\theta)}\right],
\end{align*}
 so that simple gradient-based algorithms lead to biases.

\subsection{Algorithm}
\label{ss:Algorithm}
In this paper, we will use neural networks with a single hidden layer:
\begin{equation}
\label{eq:randomFeature}
\bm{\psi}(x,\theta)=\bm{\phi}(x)^\top \theta=\sum_{i=1^m}c_i\sigma(\bm{w}_i^\top x+\bm{b}_i),
\end{equation}
where $\sigma(t)=\max\{0,t\}$ is the ReLU activation function, the weights and biases $(\bm{w}_i,\bm{b}_i)$ are drawn randomly in advance, and $\theta=\begin{bmatrix} c_1 & \cdots & c_m\end{bmatrix}^\top$ is the parameter vector. In other words, we are using a random feature method, with feature vector 
$$\bm{\phi}(x)=\begin{bmatrix} \sigma(\bm{w}_1^\top x+ \bm{b}_1) & \cdots & \sigma(\bm{w}_m^\top x+ \bm{b}_m)\end{bmatrix}^\top.
$$

With this restricted type of network, the negative of the objective from \eqref{eq:MINE} can be expressed as:
\begin{equation}
\label{eq:objective}
\bm{f}(\theta)=-\E_{\Prob}[\bm{\phi}(\bm{x})^\top\theta] +\log\left(\E_{\Q}[e^{\bm{\phi}(\bm{y})^\top \theta}] \right),
\end{equation}
where $\E_{\Prob}$ corresponds to taking expectation over  $\bm{x}$ while $\E_{\Q}$ corresponds to taking expectations over $\bm{y}$, keeping the weights and baises fixed.  
Note that $\bm{f}$ is a random function, since it depends on the random choice of weights and biases in the construction of $\bm{\phi}$. 

Let $\bm{Q}_{\theta}$ denote the probability distribution over $\Omega$ with density with respect to $\Q$ given by $\frac{d\bm{Q}_{\theta}}{d\Q}(y)=\frac{1}{\E_{\Q}[e^{\bm{\phi}(\bm{y})^\top \theta}]}e^{\bm{\phi}(y)^\top \theta}.$ Note that $\bm{Q}_{\theta}$ is a random measure, since it depends on the random function, $\bm{\phi}$. Differentiating gives:
\begin{subequations}
\begin{align}
\label{eq:exactGradient}
\nabla_{\theta}\bm{f}(\theta)&
= -\E_{\Prob}[\bm{\phi}(\bm{x})]+\E_{\bm{Q}_{\theta}}[\bm{\phi}(\bm{y})] \\
\label{eq:hessian}
\nabla^2_{\theta}\bm{f}(\theta)&=\\
\nonumber
&\hspace{-16pt}\E_{\bm{Q}_{\theta}}\left[\left(\bm{\phi}(\bm{y})-\E_{\bm{Q}_{\theta}}[\bm{\phi}(\bm{y})] \right) \left(\bm{\phi}(\bm{y})-\E_{\bm{Q}_{\theta}}[\bm{\phi}(\bm{y})] \right)^\top\right].
\end{align}
\end{subequations}

From \eqref{eq:hessian}, we see that $\bm{f}$ is convex. 

Let $\bm{\zeta}_k=(\bm{x}_k,\bm{y}_k)$ be independent samples from $\Prob\otimes \Q$. Let $\Theta$ be a compact box, to be defined later. Our algorithm is the approximate gradient descent method given by:
\begin{subequations}
\label{eq:algorithm}
\begin{align}
\bm{\theta}_{k+1}&=\Pi_{\Theta}\left(\bm{\theta}_k+\alpha r\left(\bm{\phi}(\bm{x}_k)-\frac{1}{\bm{z}_k} e^{\bm{\phi}(\bm{y}_k)^\top \bm{\theta}_k} \bm{\phi}(\bm{y}_k)\right) \right)\\
\bm{z}_{k+1}&=\bm{z}_k+\alpha\left(e^{\bm{\phi}(\bm{y}_k)^\top \bm{\theta}_k} -\bm{z}_k\right).
\end{align}
\end{subequations}
Here $\alpha>0$ is the step size for $\bm{z}_k$, $r>0$, and $\Pi_{\Theta}$ is the projection onto $\Theta$. 
The variable $\bm{z}_k$ is used to approximate the value $\E_{\Q}[e^{\bm{\phi}(\bm{y})^\top \theta}]$  in the denominator of the gradient calculation. 

Each iteration requires a single sample $\bm{\zeta}_k\in\R^{2n}$. Each entry of $\bm{\phi}(\bm{x}_k)$ and $\bm{\phi}(\bm{y}_k)$ requires $O(n)$ operations. There are $m$ entries each in $\bm{\phi}(\bm{x}_k)$ and $\bm{\phi}(\bm{y}_k)$, so that their computations require $O(mn)$ operations. The inner products require $O(m)$ operations, as  does the projection onto a box constraint. Thus, each iteration of the algorithm requires $O(mn)$ operations, where $n$ is the dimension of the random variables, $\bm{x}_k$ and $\bm{y}_k$, and $m$ is the number of neurons.

As discussed above, $\bm{f}$ is convex. Thus,
the choice of the random feature approach eliminates the problem of non-convexity that arises when using deep networks, or even just two-layer networks with trained hidden layer. The remaining challenges to analyze the algorithm become:
\begin{itemize}
\item Guarantee that with high probability, there is a $\theta$ in an appropriate set $\Theta$ such that $\bm{\phi}(x)^\top \theta \approx \log\left(\frac{d\Prob}{d\Q}(x) \right)+\xi$ for all $x\in\Omega$,
\item Bound the effect caused by using biased  gradient estimates.
\end{itemize}

\subsection{A Random Feature Approximation Result}
\label{ss:approximation}
Here we present a result on approximating smooth functions with random features. 

If $g:\R^n \to \C$, it is related to its Fourier transform $\hat g:\R^n\to \C$ by
\begin{subequations}
\label{eq:ftRelations}
\begin{align}
  \label{eq:ft}
  \hat g(\omega)&=\int_{\R^n}e^{-j2\pi \omega^\top x} g(x)dx\\
  \label{eq:ift}
  g(x)&=\int_{\R^n}e^{j2\pi \omega^\top x}\hat g(\omega)d\omega.
\end{align}
\end{subequations}
When $g\in L^1(\R^n)$ and $\hat g \in L^1(\R^n)$, these relations hold for almost all $\omega$ and $x$  in $ \R^n$.

Our approximation result extends work in \cite{lamperski2024approximation,lamperski2024functiongradientapproximationrandom}, which requires a bound on the following norm
\begin{equation}
\label{eq:Fnorm}
\|g\|_{F^{k}}=\mathrm{ess\:sup}_{\omega\in\R^n}|\hat g(\omega)|\left(1+(2\pi\|\omega\|_2)^k\right),
\end{equation}
where the essential supremum is taken with respect to the Lebesgue measure. This norm measures the smoothness of $g$, in at a bound on $\|g\|_{F^k}$ gives  a bound on $g$ and all of the derivatives of $g$ up to order $k-2$. 

The norm $\|\cdot\|_{F^k}$ was never defined explicitly in \cite{lamperski2024approximation,lamperski2024functiongradientapproximationrandom}, but an assumption equivalent to $\|g\|_{F^{k}}<\infty$ was used. Here, we also deviate  from the presentation in \cite{lamperski2024approximation,lamperski2024functiongradientapproximationrandom} by including a factor of $2\pi$ in the definition. This factor, in combination with the particular form of the Fourier  transform from \eqref{eq:ftRelations} leads to simpler expressions for the constants.

Note that $\|\cdot \|_{F^k}$ is a norm for  all $k\ge 1$. It is closely related to the Barron norm / spectral norm used in \cite{klusowski2016risk,barron1993universal}. Lemma~\ref{lem:sobolevBound} in Appendix~\ref{app:approximation} shows how $\|g\|_{F^k}$ can be bounded in terms of Sobolev norms, which give a more standard measure of smoothness.

Let $\S^{n-1}=\{x\in\R^n|\|x\|_2=1\}$ denote the $n-1$-dimensional unit sphere.
Let $A_{n-1} := \frac{2\pi^{n/2}}{\Gamma(n/2)},$ which is the surface area of the $(n-1)$-dimensional unit sphere. Let $B_R$ denote the Euclidean ball of radius $R$ around the origin. 

The proposition below gives worst-case approximation errors for approximating smooth functions with random features. Related work from \cite{lamperski2024approximation,lamperski2024functiongradientapproximationrandom} required affine terms in the neural network output (i.e. skip connections). For this paper, removal of the affine terms enables definition of a constraint set, $\Theta$, with diameter of $O(m^{-1/2})$, where $m$ is the number of neurons. This shrinking diameter simplifies the algorithmic analysis. (See Remark~\ref{rem:noAffineImprovement} for further discussion.) The proposition is proved  in Appendix~\ref{apss:approximationPf}.

\begin{proposition}
\label{prop:generalApproximation}
  For $m\ge 1$ and $R>0$,  let $\bm{w}_1$, \ldots, $\bm{w}_m$ and $\bm{b}_1$, \ldots, $\bm{b}_m$ be independent random variables  such that $\bm{w}_i$ are uniformly distributed on $\S^{n-1}$ and $\bm{b}_i$ are uniformly distributed on $[-R,R].$ 
    If $g:\R^n\to \R$ satisfies $\|g\|_{F^{n+3}}<\infty$, then
    there are coefficients $\bm{c}_1,\ldots,\bm{c}_m$ with
    \begin{align*}
      |\bm{c}_i| & \le \frac{ \left(2R + 4+3\sqrt{n}  + 4R^{-1}\right)\frac{2A_{n-1}}{(2\pi)^n}
      \|g\|_{F^{n+3}}}{m}
    \end{align*}
    such that for all $\delta \in (0,1)$, with probability at least $1-\delta$, the neural network approximation
    \begin{equation*}
    \bm{g}_N(x)=\sum_{i=1}^m\bm{c}_i \sigma(\bm{w}_i^\top x+\bm{b}_i)
    \end{equation*}
    satisfies
    \begin{gather*}
    \|\bm{g}_N-g\|_{L^{\infty}(B_R)}\le \frac{1}{\sqrt{m}}\left(\sqrt{n}+\sqrt{\log(\delta^{-1})} \right)\cdot \\
\left(
16R^2+32R+21\sqrt{n}R+36
\right)\frac{2A_{n-1}}{(2\pi)^n}\|g\|_{F^{n+3}}.
    \end{gather*}
\end{proposition}   

\begin{remark}
Approximation error bounds for random feature neural networks have been derived for a variety of metrics
\cite{gonon2023approximation,gonon2023random,neufeld2023universal,hsu2021approximation,xu2024priori}. For our purposes, it is useful to have bounds on $L^{\infty}$ errors with high probability, as given in \cite{lamperski2024approximation,lamperski2024functiongradientapproximationrandom,salanevich2023efficient}, with the current tightest bounds from \cite{lamperski2024approximation,lamperski2024functiongradientapproximationrandom}.
\end{remark}

\subsection{Main Result: Error Bounds}

\begin{assumption}
\label{as:support}
$\Prob$ and $\Q$ are supported on $\Omega\subset B_R$, where $B_R$ is ball of radius $R$ around the origin.
\end{assumption}

\begin{assumption}
\label{as:smooth}
There is a constant $\xi$ and an extension $g:\R^n\to \R$ of the function $\left(\log\left(\frac{d\Prob}{d\Q} \right)+\xi\right):\Omega\to \R$, and a number $\rho >0$ such that  
$\left\|g\right\|_{F^{n+3}}\le \rho$.
\end{assumption}

By an extension, we mean $g(x)$ is defined for all $x\in\R^n$ and that $g(x)=\log\left(\frac{d\Prob}{d\Q}(x)\right)+\xi$ for all $x\in\Omega$. The extension is needed because the norm, $\|\cdot\|_{F^{n+3}}$ is defined via the Fourier transform, which requires the function to be defined over all of $\R^n$. For reasonably simple domains, $\Omega$, e.g. convex sets, Lipschitz domains, smooth domains, classical results on Sobolev spaces guarantee that suitable extensions exist. See~\cite{adams2003sobolev,stein1970singular}. 

Motivated by Assumption~\ref{as:smooth} and Proposition~\ref{prop:generalApproximation}, we define the constant factors:
\begin{align}
\label{eq:estErrorBound}
\kappa:&=
\left(
16R^2+32R+21\sqrt{n}R+36
\right)\frac{2A_{n-1}}{(2\pi)^n}
\rho \\
C_{\Theta} &:=\left(2R + 4+3\sqrt{n}  + 4R^{-1}\right)\frac{2A_{n-1}}{(2\pi)^n}
      \rho.
\end{align}
Here, $\kappa$ bounds the estimation error over $B_R$ of any function $h$ with $\|h\|_{F^{n+3}}\le \rho$, while $C_{\Theta}/m$ bounds the size or required coefficients.


Define the constraint set for $m\ge 1$ by
\begin{equation}
\label{eq:constraints}
\Theta = \left\{\begin{bmatrix} c_1 & \cdots & c_m \end{bmatrix}^\top \middle| 
 |c_i| \le 
 \frac{C_{\Theta} }{m}
\right\}.
\end{equation}
Note that $\Theta\subset \R^{m}$ is a compact, convex set.

The following is the main result of the paper. It is proved in Appendix~\ref{app:proof}.

\begin{theorem}
\label{thm:main}
Say that Assumptions~\ref{as:support} and \ref{as:smooth} hold. Let $\overline{\bm{\theta}_T}=\frac{1}{T}\sum_{k=0}^{T-1}\bm{\theta}_k.$ For all $\delta \in (0,1)$, with probability at least $1-\delta$, the average of the iterates satisfies:
\begin{multline*}
|\E[\bm{f}(\overline{\bm{\theta}}_T)|\bm{w},\bm{b}]+D_{KL}(\Prob||\Q)|
\le \\\frac{2\kappa}{\sqrt{m}}\left(
\sqrt{n}+\sqrt{\log(\delta^{-1
})}
\right)\\+
\frac{b_1}{\alpha T}+\frac{b_2}{\alpha r T m} + b_3 \alpha r m + b_4\sqrt{\alpha},
\end{multline*}
where
\begin{align*}
b_1 &= 2RC_{\Theta}e^{8RC_{\Theta}}\\
b_2 &= \frac{C_{\Theta}^2}{2}\\
b_3 &= \left(8R^3 C_{\Theta}(e^{8R C_{\Theta}}+e^{12 R C_{\Theta}} )+2R^2(1+e^{4RC_{\Theta}})^2
\right)\\
b_4 &=2RC_{\Theta}e^{10R C_{\Theta}}.
\end{align*}

In particular, if $T\ge 2$ is fixed, the upper bound can be optimized analytically with respect to $\alpha$ and $r$ by setting:
\begin{align*}
\alpha&=2^{2/3}T^{-2/3}\\
r&=
\frac{T^{1/6}}{m } 2^{-2/3} \sqrt{\frac{b_2}{b_3}},
\end{align*}
giving the upper bound:
\begin{multline*}
|\E[\bm{f}(\overline{\bm{\theta}}_T)|\bm{w},\bm{b}]+D_{KL}(\Prob||\Q)|\le \\ 2\kappa\left(
\sqrt{n}+\sqrt{\log(\delta^{-1})}
\right)m^{-1/2}+\beta_1 T^{-1/3}+\beta_2 T^{-1/2},
\end{multline*}
where
\begin{align*}
\beta_1 &=\left(2^{-2/3}+2^{1/3}\right)b_1^{1/3}b_4^{2/3}
\\
\beta_2 &= 2\sqrt{b_2 b_3}.
\end{align*}
\end{theorem}

\begin{remark}
\label{rem:smoothness}
The constant factors, $\kappa$, $\beta_1$, and $\beta_2$ all depend on a term of the form 
\begin{equation}
\label{eq:smoothnessDimension}
\frac{A_{n-1}}{(2\pi)^n}\rho=\frac{2}{2^n\pi^{n/2}\Gamma(n/2)}\rho.
\end{equation}
In particular, $\beta_2$ and $\beta_2$ grow exponentially with this term.  Recall that $\rho$ quantifies the smoothness of $\log\left(\frac{d\Prob}{d\Q}\right)$, and is typically unknown in practice. Note further that the quantity in \eqref{eq:smoothnessDimension} decreases faster than exponential in the dimension, $n$. As a result, there is a non-trivial interplay between smoothness and dimension. See Fig.~\ref{fig:constants}. Future work will focus on deriving bounds on smoothness norm, $\|\cdot\|_{F^{n+3}}$, which was used to define the factor $\rho$, for general classes of functions.
\begin{figure}
\centering
\begin{minipage}[b]{.4\textwidth}
\includegraphics[width=\textwidth]{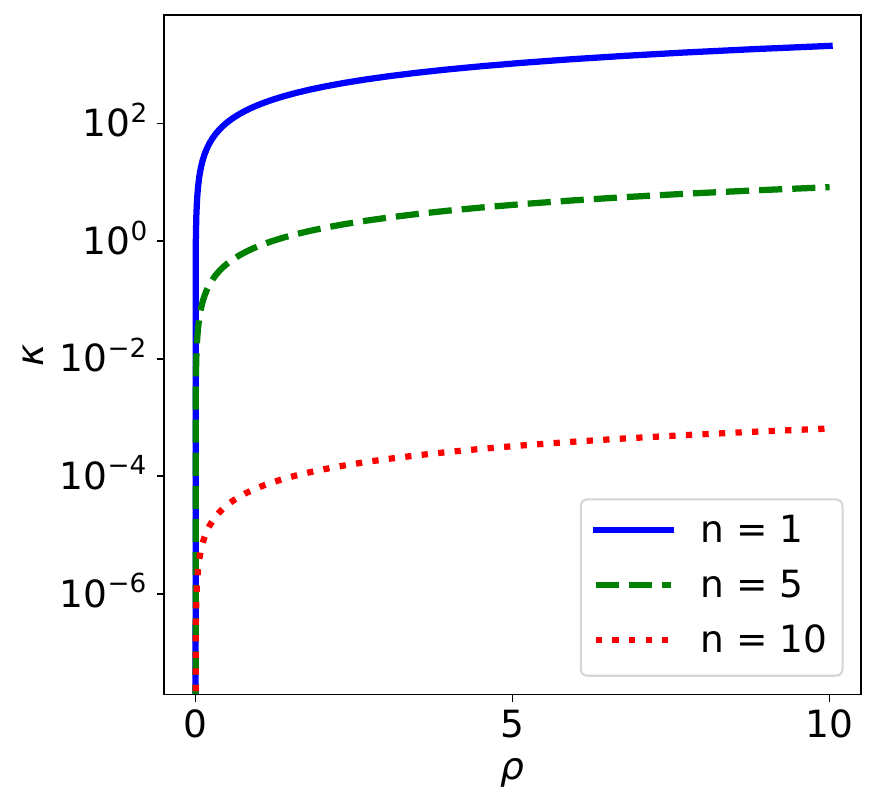}
\end{minipage}
\begin{minipage}[b]{.4\textwidth}
\includegraphics[width=\textwidth]{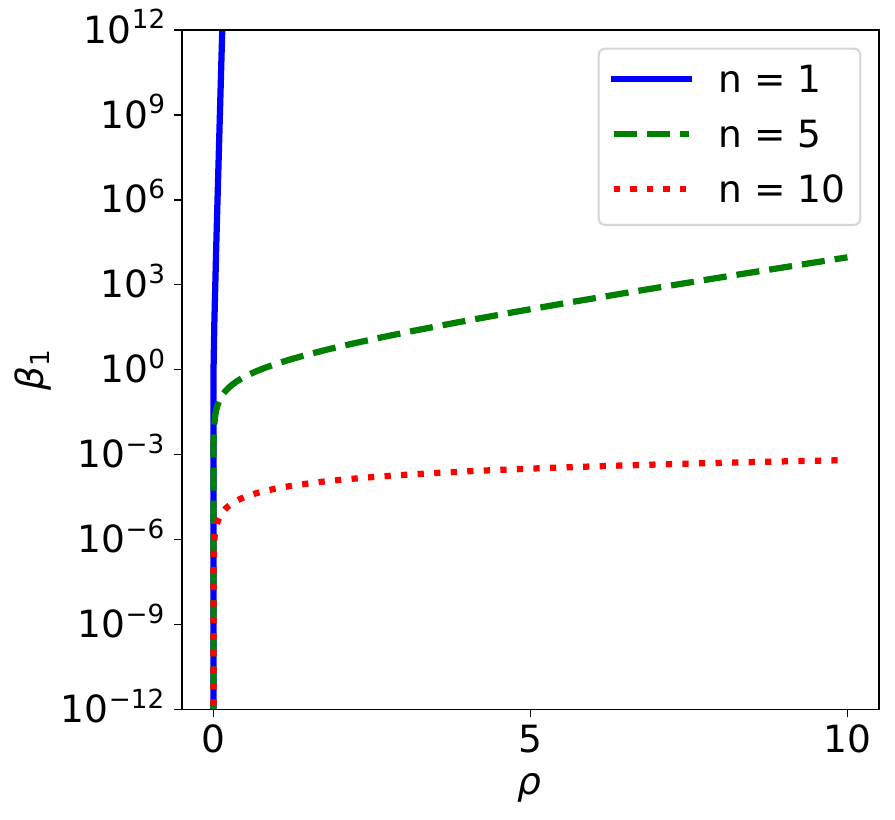}
\end{minipage}
\begin{minipage}[b]{.4\textwidth}
\includegraphics[width=\textwidth]{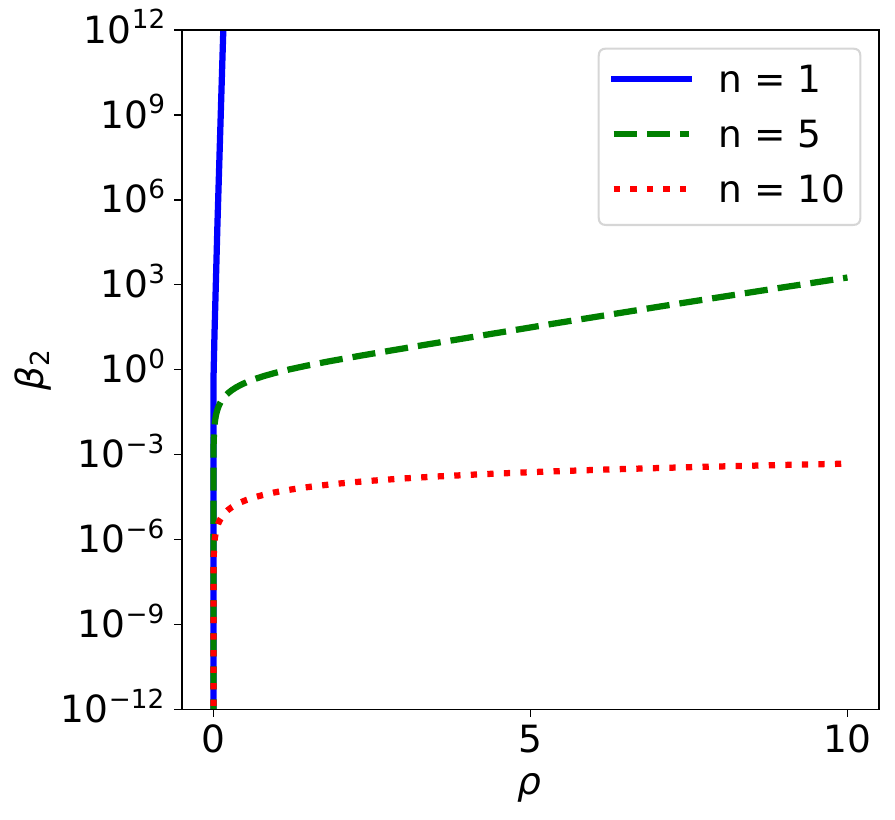}
\end{minipage}
\caption{\label{fig:constants} {\bf Smoothness and Dimension Dependence for Constant Factors.} The plots show how the constant factors, $\kappa$, $\beta_1$, and $\beta_2$ vary for different levels of the smoothness bound, $\rho$, and dimension, $n$. Note that the $y$-axes are plotted in logarithmic scales.}
\end{figure}
\end{remark}

\begin{remark}
\label{rem:noAffineImprovement}
The approximation theorem from \cite{lamperski2024approximation} uses a network of the form:
$$
\bm{g}_N(x)=a+v^\top x+\sum_{i=1}^m\bm{c}_i\sigma(\bm{w}_i^\top x+\bm{b}_i),
$$
where the bounds on $a$ and $\|v\|_2$ are  given independent of the network size, $m$. In particular, to utilize this expansion, the vector $v$ must be estimated. Appending $v$ to $\theta$ and using the bounds from the associated result in \cite{lamperski2024approximation} would result in $\Theta$ with diameter of $\Omega(1)$,  rather than $O(m^{-1/2})$ of the current paper. The decreasing diameter substantially simplifies the derivation of the final bounds for the algorithm error. 
\end{remark}

\section{Numerical Experiments}

 The link to the code for this section can be found here\footnote{https://anonymous.4open.science/r/MINEComparison-4615}.  These experiments were run on a 2020 M1 Mac with 8GB of RAM. In addition to our theoretical guarantees, we empirically evaluated the estimation algorithm on 2 examples: one with 2D distributions, and one with 5D distributions. We considered the KL divergence between a truncated multivariate Gaussian distribution and a uniform distribution, both restricted to $[-2,2]^2$ and $[-2,2]^5$. Specifically, for the 2D example let $\mathbb{P}$ be the distribution with density proportional to $\exp(-\frac{1}{2}\|x\|_2^2)$ on $[-2,2]^2$, and $\mathbb{Q}$ be the uniform distribution on the same domain. For the 5D example let $\mathbb{P}$ be the distribution with density proportional to $\exp(-\frac{1}{2}\|x\|_2^2)$ on $[-2,2]^5$, and $\mathbb{Q}$ be the uniform distribution on the same domain. We evaluate the true KL divergence in both cases using numerical integration.

 We generated random weights $\bm{w}_i$ uniformly on the unit sphere $\mathbb{S}^{1}$ and biases $\bm{b}_i$ uniformly in $[-2,2]$. Following our theoretical analysis, we set the learning rate $\alpha = T^{-2/3}$ and the parameter $r = 1/m$. The initial parameters $\bm{\theta}_0$ were sampled uniformly from $\left[\frac{-2\times 10^{6}}{\sqrt{m}}, \frac{2\times 10^6}{\sqrt{m}}\right]^m$ to ensure $\|\bm{\theta}_0\|_2 = O(1/\sqrt{m})$, and we initialized $\bm{z}_0 = 1$. The update steps follow Equation \eqref{eq:algorithm}, with projection of $\theta$ to ensure the parameters remain within the constraint set defined by \eqref{eq:constraints}. 
 
 We do two separate experiments measuring the error with respect to the number of neurons $m$ and iterations $T$. For each parameter configuration, we ran 10 independent trials. The results of these experiments are shown in \ref{fig:vary_m} and \ref{fig:vary_T} respectively. For each trial, obtain the KL divergence from the model by doing 5,000 samples from $\Prob$ and $\Q$. This yields a strong approximation of $D_{\text{KL}}^{\text{approx}} =\E[\psi(\bm{x},\theta)] - \log(\E[e^{\psi(\bm{y},\theta)}]).$

Our numerical experiments validate the practical effectiveness of the proposed algorithm. The accuracy is generally worse than the SKlearn's method. However, our method runs considerably faster than the SKlearn method for lower configurations, and were equal in runtime only at $T=5\cdot 10^6$ in 2D and $T=10^7$ in 5D. 

\begin{figure}
\centering
       \includegraphics[width=.95\columnwidth]{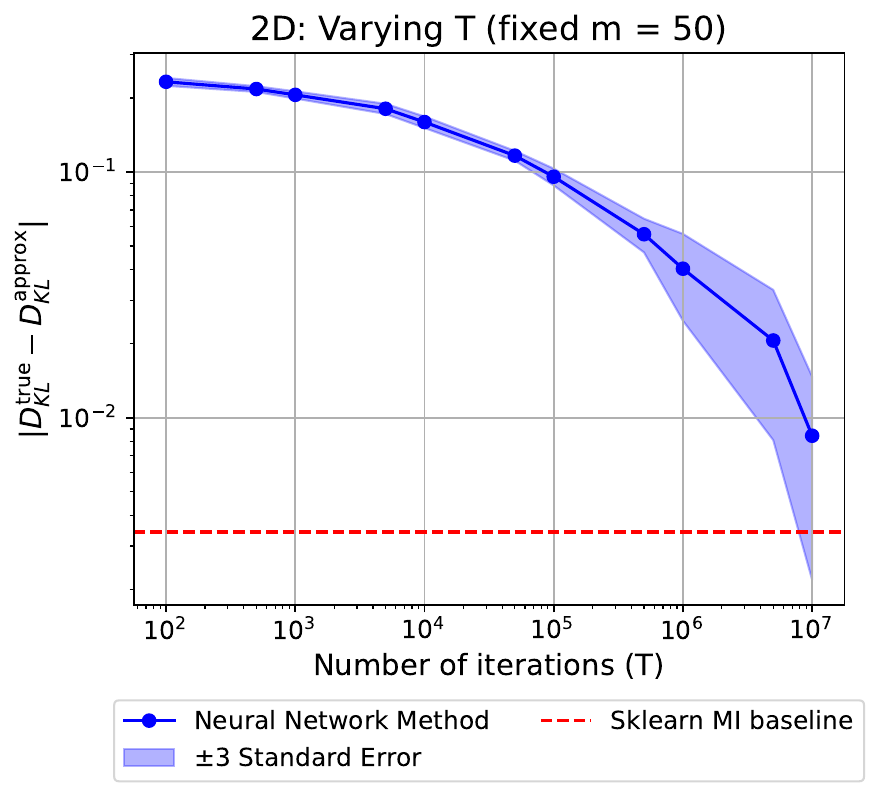}
       \caption{Scaling with iterations $T$ in the 2D case (fixed $m=50$). Error bars show $\pm$ 3 standard errors across 10 trials.}
       \label{fig:vary_T}
\end{figure}

\begin{figure}
\centering
\includegraphics[width=.95\columnwidth]{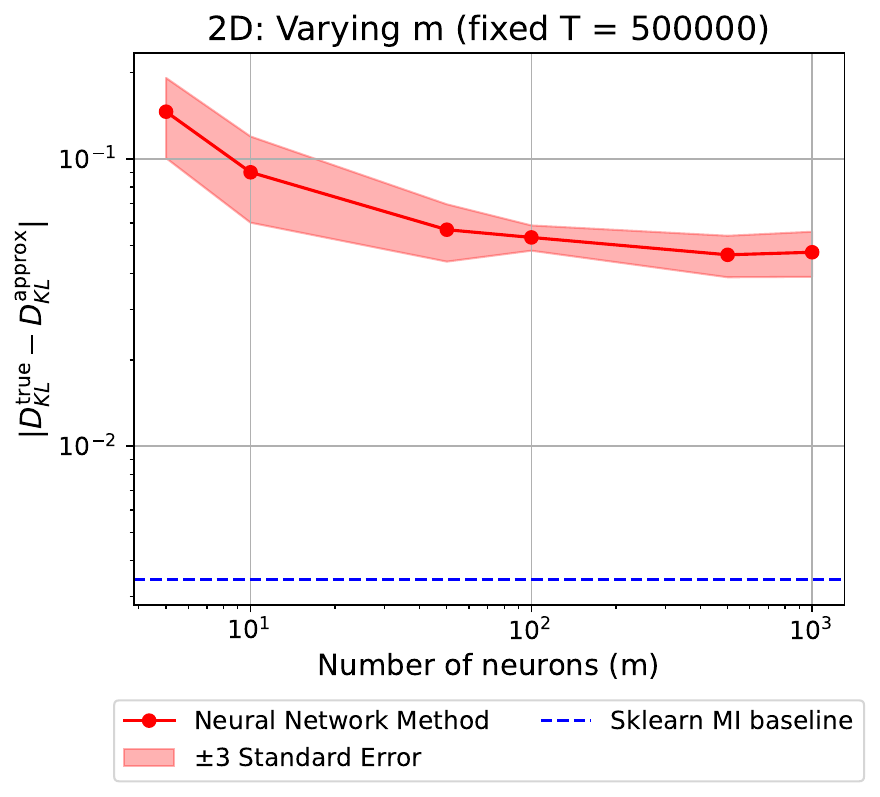}
       \caption{Scaling with network size $m$ in the 2D case (fixed $T=500,000$). Error bars show $\pm$ 3 standard errors across 10 trials.}
       \label{fig:vary_m}
\end{figure}

\begin{figure}
       \includegraphics[width=.95\columnwidth]{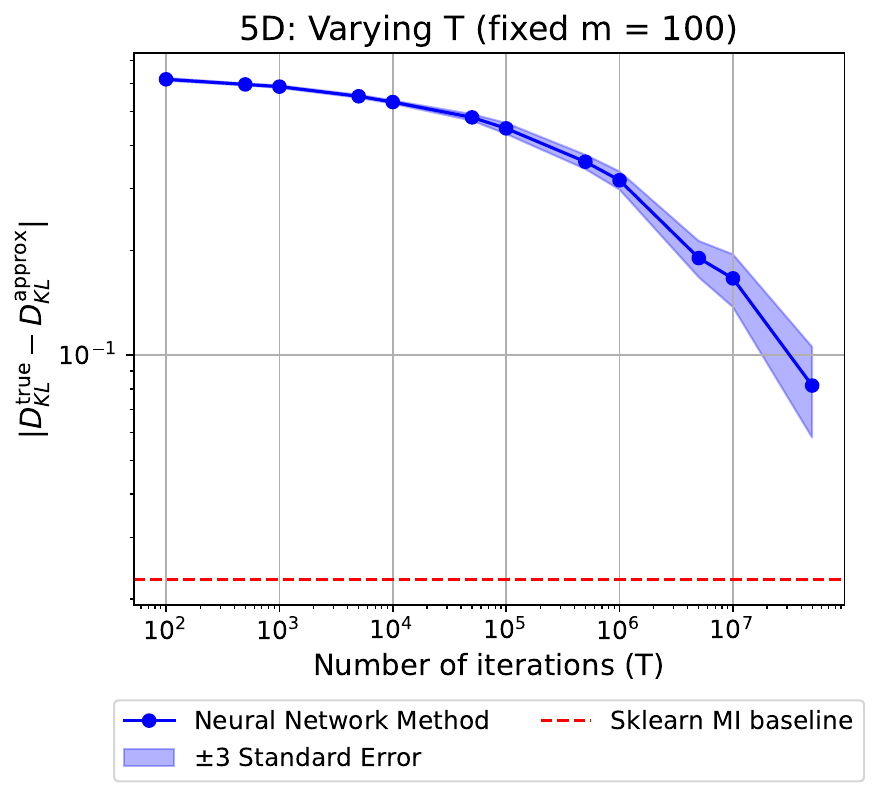}
       \caption{Scaling with iterations $T$ in the 5D case (fixed $m=100$). Error bars show $\pm$ 3 standard errors across 10 trials.}
       \label{fig:vary_T_5d}
\end{figure}

\begin{figure}
   \centering

       \includegraphics[width=.95\columnwidth]{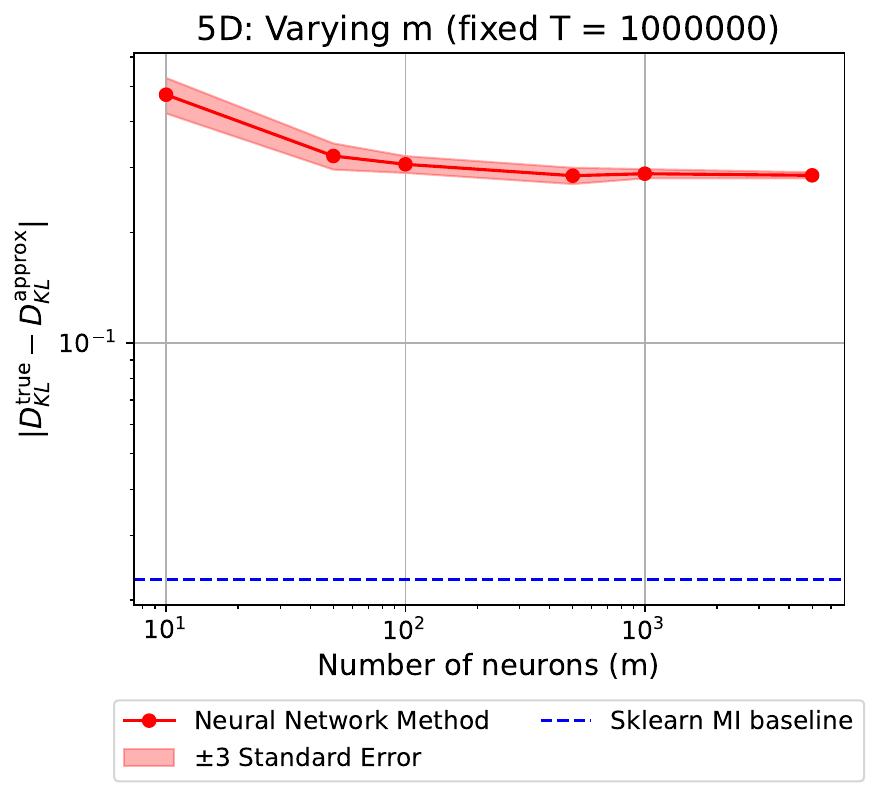}
       \caption{Scaling with network size $m$ in the 5D case (fixed $T=1,000,000$). Error bars show $\pm$ 3 standard errors across 10 trials.}
       \label{fig:vary_m_5d}
\end{figure}

\section{Conclusion}
We presented a new algorithm for estimating the KL divergence of continuous random variables  via random feature neural networks. The analyses of similar existing methods rely on non-constructive approximation theorems, and do not get bounds on the estimation error produced by the algorithms. In contrast, we give explicit quantitative error bounds on the estimation error produced by the algorithm. Future work will include extensions to data with dependencies over time, and to the use of deep neural networks for estimation.

\bibliography{main}

\if\preprint0
\section*{Checklist}

The checklist follows the references. For each question, choose your answer from the three possible options: Yes, No, Not Applicable.  You are encouraged to include a justification to your answer, either by referencing the appropriate section of your paper or providing a brief inline description (1-2 sentences). 
Please do not modify the questions.  Note that the Checklist section does not count towards the page limit. Not including the checklist in the first submission won't result in desk rejection, although in such case we will ask you to upload it during the author response period and include it in camera ready (if accepted).

\textbf{In your paper, please delete this instructions block and only keep the Checklist section heading above along with the questions/answers below.}

\begin{enumerate}

  \item For all models and algorithms presented, check if you include:
  \begin{enumerate}
    \item A clear description of the mathematical setting, assumptions, algorithm, and/or model. [Yes]
    \item An analysis of the properties and complexity (time, space, sample size) of any algorithm. [Yes]
    \item (Optional) Anonymized source code, with specification of all dependencies, including external libraries. [Yes]
  \end{enumerate}

  \item For any theoretical claim, check if you include:
  \begin{enumerate}
    \item Statements of the full set of assumptions of all theoretical results. [Yes]
    \item Complete proofs of all theoretical results. [Yes]
    \item Clear explanations of any assumptions. [Yes]     
  \end{enumerate}

  \item For all figures and tables that present empirical results, check if you include:
  \begin{enumerate}
    \item The code, data, and instructions needed to reproduce the main experimental results (either in the supplemental material or as a URL). [Yes]
    \item All the training details (e.g., data splits, hyperparameters, how they were chosen). [Yes]
    \item A clear definition of the specific measure or statistics and error bars (e.g., with respect to the random seed after running experiments multiple times). [Yes]
    \item A description of the computing infrastructure used. (e.g., type of GPUs, internal cluster, or cloud provider). [Yes]
  \end{enumerate}

  \item If you are using existing assets (e.g., code, data, models) or curating/releasing new assets, check if you include:
  \begin{enumerate}
    \item Citations of the creator If your work uses existing assets. [Not Applicable]
    \item The license information of the assets, if applicable. [Not Applicable]
    \item New assets either in the supplemental material or as a URL, if applicable. [Not Applicable]
    \item Information about consent from data providers/curators. [Not Applicable]
    \item Discussion of sensible content if applicable, e.g., personally identifiable information or offensive content. [Not Applicable]
  \end{enumerate}

  \item If you used crowdsourcing or conducted research with human subjects, check if you include:
  \begin{enumerate}
    \item The full text of instructions given to participants and screenshots. [Not Applicable]
    \item Descriptions of potential participant risks, with links to Institutional Review Board (IRB) approvals if applicable. [Not Applicable]
    \item The estimated hourly wage paid to participants and the total amount spent on participant compensation. [Not Applicable]
  \end{enumerate}

\end{enumerate}
\fi

\clearpage

\appendix
\thispagestyle{empty}

\onecolumn

\section{Elementary Background Results}
This appendix collects some elementary results and facts that are used to prove the approximation result, Proposition~\ref{prop:generalApproximation}.

\subsection{Integration on Spheres}

For $n\ge 2$, the spherical coordinate representation from \cite{blumenson1960derivation} is given by
\begin{equation}
\label{eq:spherical}
w=h(\phi)=\begin{bmatrix}
\cos(\phi_1) \\
\cos(\phi_2)\sin(\phi_1) \\
\vdots \\
\cos(\phi_{n-2})\prod_{i=1}^{n-3}\sin(\phi_i) \\
\sin(\phi_{n-1})\prod_{i=1}^{n-2}\sin(\phi_i)\\
\cos(\phi_{n-1})\prod_{i=1}^{n-2}\sin(\phi_i)
\end{bmatrix},
\end{equation}
where we use the convention that $\prod_{i=1}^k\sin(\phi_k)=1$ if $k\le 0.$ The angle parameters are given by $\phi \in \Phi:=[0,\pi]^{n-2}\times [0,2\pi)$. In particular, when $n=2$, the representation reduces to 
$$
h(\phi)=\begin{bmatrix}
\sin(\phi_1)\\
\cos(\phi_1)
\end{bmatrix}.
$$

Let $Dh(\phi)$ denote the Jacobian matrix of $h$. Let $\mu_{n-1}$ denote the $(n-1)$-dimensional Hausdorff measure over $\R^n$.

\begin{lemma}
\label{lem:coordinateIntegral}
If $f\in L^1(\S^{n-1})$, then its integral can be expressed in the following equivalent ways:
\begin{align*}
\int_{\S^{n-1}}f(w)\mu_{n-1}(dw)&=\int_{\Phi}
f(h(\phi))\sqrt{\det(Dh(\phi)^\top Dh(\phi))}d\phi \\
&=\int_{\Phi}f(h(\phi))\left(\prod_{i=1}^{n-2}\sin^{n-1
-i}(\phi_i)\right)d\phi.
\end{align*}
\end{lemma}

\begin{proof}
The first equality follows from applying Theorem 11.25 from \cite{folland1999real}, which shows how to evaluate integrals with respect to Hausdorff measures via parameterizations. 

Proving the second equality amounts to showing that 
\begin{equation}
\label{eq:determinantForm}
\sqrt{\det(Dh(\phi)^\top Dh(\phi))}=\prod_{i=1}^{n-2}\sin^{n-1-i}(\phi_i).
\end{equation}
As discussed in \cite{blumenson1960derivation}, 
$$
\det\begin{bmatrix}
h(\phi) & Dh(\phi)
\end{bmatrix}
=\prod_{i=1}^{n-2}\sin^{n-1-i}(\phi_i).
$$
Then, using that $h(\phi)^\top h(\phi)=1$ and $h(\phi)^\top Dh(\phi)=0$ gives:
\begin{equation*}
\begin{bmatrix}
h(\phi) & Dh(\phi)
\end{bmatrix}^\top \begin{bmatrix}
h(\phi) & Dh(\phi)
\end{bmatrix}=\begin{bmatrix}
1 & 0 \\
0 & Dh(\phi)^\top Dh(\phi)
\end{bmatrix}
\end{equation*}
Thus, \eqref{eq:determinantForm} follows by taking the determinant of  this matrix and then applying the square root.
\end{proof}

The following is an elementary observation about rotational invariance of integrals over $\S^{n-1}$. 

\begin{lemma}
\label{lem:rotationalInvariance}
If $f:\S^{n-1}\to \C$ is in $L^1(\S^{n-1})$ and $U$ is an $n\times n$ orthogonal matrix, then
$$
\int_{\S^{n-1}}f(w)\mu_{n-1}(dw)=\int_{\S^{n-1}}f(Uz)\mu_{n-1}(dz).
$$
\end{lemma}

\begin{proof}
Let $z = U^\top w$, so that $w=Uz$. Let $z=h(\psi)$,  using the spherical coordinate parameterization for $z$, leading to an alternative parameterization, $w=Uh(\psi)$. The Jacobian matrix of this parameterization for $w$ is $UDh(\psi)$. Orthogonality of $U$ implies that $(UDh(\psi))^\top (UDh(\psi))=(Dh(\psi))^\top (Dh(\psi))$. Using Theorem~11.25 of \cite{folland1999real} then
gives:
\begin{align*}
\int_{\S^{n-1}} f(w)\mu_{n-1}(dw)&=\int_{\Phi}f(Uh(\psi))\sqrt{\det(Dh(\psi)^\top Dh(\psi))}d\psi \\
&=\int_{\S^{n-1}}f(Uz)\mu_{n-1}(dz). 
\end{align*}
\end{proof}

The following result is a special case of the discussion of integration from \cite{blumenson1960derivation}. 
\begin{lemma}
\label{eq:1dIntegral}
If $n\ge 3$,  $g\in L^1(\R)$ and $v\in \R^n$, then
$$
\int_{\S^{n-1}}g(v^\top w)\mu_{n-1}(dw)=A_{n-2}\int_0^\pi g(\|v\|_2 \cos(\phi_1))\sin(\phi_1)^{n-2}d\phi_1.
$$
\end{lemma}

\subsection{A Variation on the Dudley Entropy Integral Bound}

If $\cX$ is a set with a metric $d$, and $\epsilon >0$,let $N(\epsilon,\cX,d)$ denote the associated \emph{covering number}. In other words, $N(\epsilon,\cX,d)$ denotes the smallest number of $d$-balls of radius $\epsilon$ required to cover $\cX$.

The following is a variation on the Dudley entropy integral bound in which bounds the effect of truncating the upper tail. A more common variation, as in \cite{wainwright2019high}, truncates the lower tail. The almost sure Lipschitz assumption is used to avoid technicalities about suprema over infinite sets,  and can likely be relaxed.

\begin{lemma}
\label{lem:dudley}
Let $\bm{f}$ be a stochastic process over an index set $\cX$ and let $d$ be a metric over $\cX$ such that:
\begin{itemize}
\item $\bm{f}(x)$ is $L$-Lipschitz with respect to $d$ almost surely
\item $\bm{f}(x)$ is zero-mean and $\tau$-sub-Gaussian for all $x\in\cX$
\item $(\bm{f}(x)-\bm{f}(y))$ is  $d(x,y)$-sub-Gaussian for all $x,y\in\cX$
\end{itemize}
For all $\epsilon >0$
$$
\E\left[\sup_{x\in\cX} \bm{f}(x)\right]\le \tau\sqrt{2\log(N(\epsilon/2,\cX,d))}+
4\int_{0}^{\epsilon} \sqrt{2\log(N(t,\cX,d)}dt.
$$
\end{lemma}

\begin{proof}
If $\cX$ is not bounded with respect to metric $d$, then the right side of the inequality is infinite, and so the bound holds automatically.

Assume that $\cX$ is bounded with respect to $d$, and let $D$ be the corresponding diameter. 

For integers, $i\ge 0$, let $\cU_i$ be a $(D2^{-i})$-covering of $\cX$ of minimal size, so that $|\cU_i|=N(D2^{-i},\cX,d).$ Let $\pi_i:\cX\to \cU_i$ be a mapping of the form:
$$
\pi_i(x)=\mathrm{arg\:min}_{y\in\cU_i} d(x,y).
$$
Note that for all $x\in \cX$, $d(x,\pi_i(x))\le D2^{-i}.$

Let $0\le i_0 < M$ be integers. For all $x\in\cX$, set $y_M(x)=\pi_M(x)$ and for $i=M-1,\ldots,i_0$, set $y_i(x)=\pi_{i}(y_{i+1}(x)).$ Then
\begin{align*}
\bm{f}(x)&=\left(\bm{f}(x)-\bm{f}(y_M(x))\right)+\bm{f}(y_M(x))\\
&=\left(\bm{f}(x)-\bm{f}(y_M(x))\right)+\bm{f}(y_{i_0}(x)) + \sum_{i=i_0}^{M-1}\left(
\bm{f}(y_{i+1}(x))-\bm{f}(y_{i}(x))
\right).
\end{align*}

Then using the almost sure Lipschitz property, the following bound holds almost surely:
\begin{align*}
\sup_{x\in\cX}\bm{f}(x)&\le LD 2^{-M}+\max_{u_{i_0}\in\cU_{i_0}}\bm{f}(u_{i_0})
+\sum_{i=i_0+1}^M\max_{u_i\in\cU_i}\left(\bm{f}(u_i)-\bm{f}(\pi_{i-1}(u_i))\right).
\end{align*}

Using a standard bound on the maxima of a finite set of sub-Gaussian random variables, e.g. Exercise 2.12 of \cite{wainwright2019high}, gives the bound in expectation: 
\begin{align*}
\E\left[
\sup_{x\in\cX}\bm{f}(x)
\right]&\le LD2^{-M}+\tau \sqrt{2\log(N(D2^{-i_0},\cX,d))}
+\sum_{i=i_0+1}^M D2^{-i+1}\sqrt{2\log(N(D2^{-i},\cX,d))}
\end{align*}
Using that $N(t,\cX,d)$ is non-increasing gives:
$$
D2^{-i-1} \sqrt{2\log(N(D2^{-i},\cX,d))}\le \int_{D2^{-i-1}}^{D2^{-i}}
\sqrt{2\log(N(tt,\cX,d))}dt
$$
for all $i$. 

Plugging in this integral bound gives
$$
\E\left[
\sup_{x\in\cX}\bm{f}(x)
\right]\le LD2^{-M}+\tau \sqrt{2\log(N(D2^{-i_0},\cX,d))}
+4\int_{D2^{-M-1}}^{D2^{-i_0}}
\sqrt{2\log(N(t,\cX,d))}dt.
$$

This bound holds for all integers $0\le i_0<M$. Letting $M\to \infty$ gives
$$
\E\left[
\sup_{x\in\cX}\bm{f}(x)
\right]\le \tau \sqrt{2\log(N(D2^{-i_0},\cX,d))}
+4\int_0^{D2^{-i_0}}
\sqrt{2\log(N(t,\cX,d))}dt.
$$

For any $\epsilon >0$, let $i_0$ be such that $D2^{-i_0}\le \epsilon \le D2^{-i_0+1}$. Then $\epsilon/2\le D2^{-i_0}$, and the result follows because $N(t,\cX,d)$ is non-increasing in $t$, while the integral term is non-decreasing in the upper limit. 
\end{proof}

\section{Smooth Functions and Approximation}
\label{app:approximation}

This appendix gives background and results on approximating smooth functions via random features. Relations between our smoothness measure, $\|\cdot\|_{F^k}$ and Sobolev norms are given in Subsection~\ref{appss:smoothness}. The approximation result, Proposition~\ref{prop:generalApproximation} is proved in Subsections~\ref{appss:integral} and \ref{apss:approximationPf}.

\subsection{Fourier Transforms and Smooth Functions}
\label{appss:smoothness}

Relating the $F^k$-norms to $L^1$ and $W^{k,1}$ norms requires some notation about the  unit sphere. 
Let $\S^{n-1}=\{x\in\R^n|\|x\|_2=1\}$ denote the $n-1$-dimensional unit sphere.
We denote the area of the area of $\S^{n-1}$ by:
\begin{equation*}
  A_{n-1} = \frac{2\pi^{n/2}}{\Gamma(n/2)},
\end{equation*}
where $\Gamma$  is the gamma function. 

Recall that $\mu_{n-1}$ denotes the $(n-1)$-dimensional Haussdorff measure over $\R^n$, so that $A_{n-1}=\int_{\S^{n-1}}\mu_{n-1}(d\alpha)$. In particular, $\S^{0}=\{-1,1\}$ and $\mu_{0}$ is the counting measure, with $\mu_{0}(\{-1\})=\mu_{0}(\{1\})=1$.

\begin{lemma}
\label{lem:sobolevBound}
For all $k\ge 1$, if $g\in W^{k,1}(\R^n)$, then
$\|g\|_{F^k}\le \max\left\{1,n^{\frac{k}{2}-1} \right\}\|g\|_{W^{k,1}(\R^n)}$.
\end{lemma}

\begin{proof}
For $\alpha = (\alpha_1,\ldots,\alpha_n)\in\N^n$, let $D^{\alpha}g=\frac{\partial^{\alpha_1}\cdots \partial^{\alpha_n}g}{\partial x_1^{\alpha_1} \cdots \partial x_n^{\alpha_n}}$.

The derivative formula for Fourier transforms gives 
$$
\widehat{D^{\alpha} g}(\omega) = (j2\pi)^{|\alpha|} \omega^{\alpha} \hat g(\omega),
$$
where $\omega^{\alpha}=\omega_1^{\alpha_1}\cdots \omega_n^{\alpha_n}.$ See, e.g., \cite{kammler2007first}. (This formula remains valid almost everywhere when $D^{\alpha}g$ are  weak derivatives.)

It follows from the Fourier transform formula, \eqref{eq:ft}, that
$$
\mathrm{ess\:sup}_{\omega\in\R^n}|\hat g(\omega)|(2\pi)^{|\alpha|} |\omega^{\alpha} |\le \|D^{\alpha}g\|_{L^1(\R^n)}
$$

A standard relationship between $p$-norms gives:
\begin{equation}
\label{eq:pNormRelations}
\|\omega\|_2 \le \begin{cases}n^{\frac{1}{2}-\frac{1}{k}}\|\omega\|_k & k\ge 2 \\
\|\omega\|_1 & k=1
\end{cases}
\end{equation}
See, e.g., \cite{horn2012matrix}.

\begin{align*}
\|g\|_{W^{k,1}(\R^n)}&\ge \left(1 + (2\pi)^k\sum_{i=1}^n |\omega_i|^k\right)|\hat g(\omega)| \\
&=\left(1 + (2\pi)^k\|\omega\|_k^k\right)|\hat g(\omega)|
\end{align*}

For $k=1$, \eqref{eq:pNormRelations} gives, almost everywhere
\begin{align*}
\|g\|_{W^{1,1}}
&\ge\left(1 + (2\pi) \|\omega\|_2\right)|\hat g(\omega)|
\end{align*}
So, at $k=1$, we have
$$
\|g\|_{F^1}\le \|g\|_{W^{1,1}}
$$

For $k\ge 2$, \eqref{eq:pNormRelations} implies that $\|\omega\|_k^k\ge n^{1-\frac{k}{2}} \|\omega\|_2^k$. Note that $1-\frac{k}{2}\le 0$, so that $n^{1-\frac{k}{2}}\le 1.$
Thus, we have, almost everywhere
\begin{align*}
\|g\|_{W^{k,1}}
&\ge\left(1 + (2\pi)^k n^{1-\frac{k}{2}} \|\omega\|_2^k\right)|\hat g(\omega)| \\
&\ge  n^{1-\frac{k}{2}} \left(1+(2\pi\|\omega\|_2)^k\right)|\hat g(\omega)|,
\end{align*}
so that in this case
$$
\|g\|_{F^k}\le n^{\frac{k}{2}-1}\|g\|_{W^{k,1}}.
$$
Combining the bounds gives the general upper bound on $\|g\|_{F^k}$
\end{proof}

\subsection{An Integral Representation for Smooth Functions}
\label{appss:integral}
Lemma~\ref{lem:integral}, below, is a modification of a result from \cite{lamperski2024approximation}, and forms the basis of the corresponding approximation result. It shows that that any sufficiently smooth function can be represented can be represented via an integral of the ReLU activation function over $\S^{n-1}\times [-R,R]$ and an affine term. Approximation schemes based on this result require an affine term. For algorithm of this paper, the affine term complicates the analysis. This subsection gives an alternative integral representation with no affine term. 

The form of Lemma \ref{lem:integral} is slightly different from the statement from \cite{lamperski2024approximation}. The biggest difference is that we utilize a slightly different measure of smoothness, from \eqref{eq:Fnorm}, which ends up simplifying the constants.

\begin{lemma}
  \label{lem:integral}
  {\it
    Let $g:\R^n\to \R$ satisfy $\|g\|_{F^{n+3}}<\infty$. For any $R>0$,
    there is a function $\xi:\S^{n-1}\times [-R,R]\to \R$, a vector $v\in\R^n$, and a scalar $r\in\R$ such that for almost all $\|x\|_2\le R$
  \begin{equation*}
    g(x)=\int_{-R}^R\int_{\S^{n-1}} \xi(w,b)\sigma(w^\top x+b)\mu_{n-1}(dw)db
        +v^\top x +r 
    \end{equation*} 
Furthermore, $\xi$, $v$, and $r$ satisfy:
\begin{align*}
   \|\xi\|_{L^{\infty}(\S^{n-1}\times [-R,R])}&\le \frac{2}{(2\pi)^n}\|g\|_{F^{n+3}} \\
  \|v\|_2 & \le \frac{2A_{n-1}}{(2\pi)^n} \|g\|_{F^{n+3}} \\
  |r|&\le (R+1)\frac{2A_{n-1}}{(2\pi)^n}\|g\|_{F^{n+3}}.
\end{align*}
  }
\end{lemma}

\begin{proof}
The main difference between this result and the corresponding result from \cite{lamperski2024approximation} is the inclusion of the $2\pi$ factor in the definition of $\|\cdot\|_{F^{n+3}}$. So, the argument from \cite{lamperski2024approximation} will be sketched briefly, mostly to show how the constant factors change. 

From here, we get that for all $i=0,1,2$:
\begin{align}
\nonumber
\int_{\R^n}|\hat g(\omega)|\cdot \|2\pi\omega\|_2^i d\omega &\le
\|g\|_{F^{n+3}}\int_{\R^n}\frac{\|2\pi\omega\|^i}{1+\|2\pi \omega\|^{n+3}}d\omega \\
\nonumber
&\overset{u=2\pi \omega}{=}\frac{\|g\|_{F^{n+3}}}{(2\pi)^n}\int_{\R^n}\frac{\|u\|_2^i}{1+\|u\|_2^{n+3}}du\\
\nonumber
&=\frac{\|g\|_{F^{n+3}}A_{n-1}}{(2\pi)^n}\int_0^{\infty}\frac{r^{i+n-1}}{1+r^{n+3}}dr\\
\label{eq:L1Bounds}
&\le \frac{2 A_{n-1}}{(2\pi)^n}\|g\|_{F^{n+3}}. 
\end{align}
The second equality uses integration in spherical coordinates. 

In particular, $\|\hat g\|_{L^1(\R^n)}<\infty$, so that the inverse Fourier transform relation, \eqref{eq:ift}, must hold for almost all $x\in\R^n$.

Let $\hat g(\omega)=|\hat g(\omega)|e^{j2\pi\theta(\omega)}$ be the magnitude and phase representation of $\hat g(\omega)$.

Set:
\begin{align*}
Z&=\int_{\R^n}|\hat g(\omega)|\cdot \|2\pi \omega\|_2^2 d\omega \\
p(\omega)&=\frac{|\hat g(\omega)|\cdot \|2\pi \omega\|_2^2}{Z} \\
\psi(t,\omega)&=\frac{Z}{\|2\pi\omega\|_2^2}\cos\left(2\pi(\|\omega\|_2 t+\theta(\omega))\right),
\end{align*}
where $\psi$ is defined for $(t,\omega)\in\R\times (\R^n\setminus\{0\})$. Here $p$ defines a probability density over $\R^n$.

Then, the calculation in \cite{lamperski2024approximation} shows that
for almost all $x\in B_R$, 
\begin{multline}
f(x)=\left(\int_{\R^n}\frac{\partial\psi(-R,\omega)}{\partial t}\frac{\omega}{\|\omega\|_2}p(\omega) d\omega \right)^\top x +
\int_{\R^n}\left(\frac{\partial \psi(-R,\omega)}{\partial t}R+\psi(-R,\omega) \right)p(\omega)d\omega +
\\
\int_{\R^n}\int_{-R}^R\frac{\partial^2\psi(t,\omega)}{\partial t^2}\sigma\left(
\left(\frac{\omega}{\|\omega\|_2}\right)^\top x-t
\right)dt p(\omega) d\omega.
\end{multline}
The first two terms define $v$ and $r$, respectively.

Now we can bound $\|v\|_2$ and $|r|$:
\begin{align*}
\|v\|_2&\le\int_{\R^n}\left| \frac{\partial\psi(-R,\omega)}{\partial t}\right| p(\omega)d\omega \\
&\le \int_{\R^n}|\hat g(\omega)|\cdot \|2\pi \omega\|_2 d\omega \\
&\overset{\eqref{eq:L1Bounds}}{\le} \frac{2 A_{n-1}}{(2\pi)^{n}}\|g\|_{F^{n+3}},
\end{align*}
and
\begin{align*}
|r|&\le \int_{\R^n}
\left(\left|\frac{\partial \psi(-R,\omega)}{\partial t}\right|R+|\psi(-R,\omega)| \right)p(\omega)d\omega\\
&\le \int_{\R^n}
\left(|\hat g(\omega)|\cdot \|2\pi\omega\|_2 R + |\hat g(\omega)|\right)d\omega \\
&\overset{\eqref{eq:L1Bounds}}{\le} \left(R+1\right)\frac{2 A_{n-1}}{(2\pi)^n}\|g\|_{F^{n+3}}. 
\end{align*}

For $\alpha \in \S^{n-1}$, set 
$$
\xi(\alpha,-t)=\int_{0}^\infty \frac{\partial^2 \psi(t,r\alpha) }{\partial t^2}r^{n-1}p(r\alpha) dr.
$$
Then, using spherical coordinates, and Fubini's theorem:
\begin{align*}
\MoveEqLeft
\int_{\R^n}\int_{-R}^R\frac{\partial^2\psi(t,\omega)}{\partial t^2}\sigma\left(
\left(\frac{\omega}{\|\omega\|_2}\right)^\top x-t
\right)dt p(\omega) d\omega \\
&=
\int_{\S^{n-1}}\int_{-R}^R\xi(\alpha,-t)\sigma(\alpha^\top x-t)dt \mu_{n-1}(d\alpha)\\
&\overset{b=-t}{=}\int_{\S^{n-1}}\int_{-R}^R\xi(\alpha,b)\sigma(\alpha^\top x+b)db \mu_{n-1}(d\alpha)\\
\end{align*}
In particular, this shows that the stated integral representation holds.

Now, we must bound $\|\xi\|_{L^{\infty}(\S^{n-1}\times [-R,R])}$:
\begin{align*}
|\xi(\alpha,-t)|&\le 
\int_0^{\infty}\left|\frac{\partial^2 \psi(t,r\alpha) }{\partial t^2}\right|r^{n-1}p(r\alpha) dr \\
&\le \int_0^{\infty}r^{n-1}|\hat g(r\alpha)|\cdot \|2\pi r\alpha\|_2^2 dr\\
&\le \|g\|_{F^{n+3}}\int_0^{\infty}\frac{(2\pi)^2 r^{n+1}}{1+(2\pi r)^{n+3}}dr\\
&\overset{u=2\pi r}{=}\frac{\|g\|_{F^{n+3}}}{(2\pi )^n}\int_0^{\infty}\frac{u^{n+1}}{1+u^{n+3}}du \\
&\le \frac{2}{(2\pi)^n}\|g\|_{F^{n+3}}.
\end{align*}
\end{proof}

Let $\mathrm{sign}$ denote the sign function:
$$
\mathrm{sign}(t)=\begin{cases}
1  & t >0 \\
0 & t = 0 \\
-1 & t < 0
\end{cases}
$$

\begin{lemma}
For $R>0$ and $r\in\R$, the function $s:[-R,R]\to \R$ defined by $s(b)=\frac{r}{R^2}\sign(b)$ is an optimal solution to the following functional optimization problem:
\begin{align*}
&\min_{f} && \|f\|_{L^{\infty}([-R,R])}\\
&\textrm{subject to} && \int_{-R}^{R}f(b)bdb =r. 
\end{align*}
\end{lemma}

\begin{proof}
For every $r\in\R$, feasibility of $s$ follows from direct calculation. Note that the value achieved is $|r|/R^2$ To prove that $s$ is optimal, we construct the Lagrange dual and show that the dual also achieves a value of $|r|/R^2$.

The optimization problem is equivalent to the following linear program over $(t,f)\in \R\times L^{\infty}([-R,R])$:
\begin{align*}
& \min_{t,f} && t \\
&\textrm{subject to} && -t \le f(b) \le t \textrm{ for almost all } b\in [-R,R]\\
&&& \int_{-R}^Rf(b)b db = r.
\end{align*}

The Lagrangian is given by:
\begin{align*}
L(t,f,\lambda,\alpha,\beta)&=t+\lambda r+\int_{-R}^{R}\left(-\alpha(b)(t+f(b))+\beta(b)(f(b)-t)-\lambda f(b)b\right)db\\
&=t\left(1-\int_{-R}^{R}(\alpha(b)+\beta(b))db\right)+\int_{-R}^{R}f(b)\left(\beta(b)-\alpha(b)-\lambda b \right)db,
\end{align*}
where $(\lambda,\alpha,\beta)\in \R\times L^1([-R,R])\times L^1([-R,R]).$

The corresponding dual problem is 
\begin{align*}
& \max_{\lambda,\alpha,\beta} && \lambda r \\
&\textrm{subject to} && \alpha(b)\ge 0, \beta(b)\ge 0 \textrm{ for almost all } b\in [-R,R] \\
&&& \int_{-R}^R(\alpha(b)+\beta(b))db = 1 \\
&&& \beta(b)-\alpha(b)=\lambda b \textrm{ for almost all } b\in [-R,R].
\end{align*}

For $r=0$, the only possible dual value is $0$, which matches the corresponding primal value. When $r\ne 0$, we can set
\begin{align*}
\lambda &= \sign(r)/R^2 \\
\alpha(b) &= \begin{cases}
-\lambda b & \lambda b < 0 \\
0 & \lambda b \ge 0 
\end{cases} \\
\beta(b) &= \begin{cases}
\lambda b & \lambda b > 0 \\
0 & \lambda b \le 0.
\end{cases}
\end{align*}
The, by construction, $\alpha(b)+\beta(b)=|\lambda b|$, $(\lambda,\alpha,\beta)$ is dual feasible, and achieves the value of $|r|/R^2$, which matches the primal value. Thus, $s(b)=r\sign(b)/R^2$ is optimal, by  weak duality. 
\end{proof}

\begin{corollary}
\label{cor:constRep}
    For $r\in \R$ and $R>0$ let $s(b)=\frac{r}{R^2
    q}\sign(b)$. The function $\zeta:[-R,R]\to \R$ defined by 
    $\zeta(b)=\frac{2}{A_{n-1}}s(b)$ satisfies $\|\zeta\|_{L^{\infty}([-R,R])}=\frac{2|r|}{A_{n-1}R^2}$ and
    $$
    \int_{-R}^R\int_{\S^{n-1}}\zeta(b)\sigma(w^\top x+b)\mu_{n-1}(dw)db = r. 
    $$
\end{corollary}

\begin{proof}
The value of $\|\zeta\|_{L^{\infty}([-R,R])}$ follows by construction. 

Using the identity $t=\sigma(t)-\sigma(-t)$ gives
\begin{align*}
r&=\frac{1}{A_{n-1}}\int_{-R}^R\int_{\S^{n-1}}s(b)(w^\top x+b)\mu_{n-1}(dw)db \\
&=\frac{1}{A_{n-1}}\int_{-R}^R\int_{\S^{n-1}}s(b)\left(\sigma(w^\top x+b)-\sigma(-w^\top x -b)\right)\mu_{n-1}(dw)db.
\end{align*}

Using the change of coordinates $\hat w=-w$ and $\hat b=-b$, along with Lemma~\ref{lem:rotationalInvariance} gives:
\begin{equation*}
\int_{-R}^R\int_{\S^{n-1}}s(b)\sigma(-w^\top x -b)\mu_{n-1}(dw)db
=\int_{-R}^R\int_{\S^{n-1}}s(-\hat b)\sigma(\hat w^\top x +\hat b)\mu_{n-1}(d\hat w)d\hat b.
\end{equation*}
Plugging this equality result into the previous equality gives:
\begin{equation*}
r=\frac{1}{A_{n-1}}\int_{-R}^R\int_{\S^{n-1}}\left(s(b)-s(-b)\right)\sigma(w^\top x+b)\mu_{n-1}(dw)db.
\end{equation*}
The result now follows after noting that $\sign(b)-\sign(-b)=2\sign(b)$ for all $b\in\R$.
\end{proof}

\begin{lemma}
\label{lem:optLinear}
For $n\ge 1$ and $v\in \R^n$,
the function $q:\S^{n-1}\to \R$ defined by $q(w)=\frac{\|v\|_2}{H_{n-1}}\sign(v^\top w)$, where $H_{n-1}=\frac{2\pi^{\frac{n-1}{2}}}{\Gamma\left(\frac{n+1}{2}\right)}$, is an optimal solution to the following functional optimization problem:
\begin{align*}
&\min_f && \|f\|_{L^{\infty}(\S^{n-1})}\\
&\textrm{subject to} && \int_{\S^{n-1}}w f(w)\mu_{n-1}(dw)=v.
\end{align*}
\end{lemma}

\begin{proof}
For $v=0$, the function becomes $q(w)=0$ for all $w\in \S^{n-1}$, which is feasible and achieves the smallest possible norm. Thus the lemma holds in this case. The rest of the proof will focus on the case that $v\ne 0$.

The proof proceeds as follows. We show that $q(w)=\frac{\|v\|_2}{H_{n-1}}\sign(v^\top w)$ is feasible. Note  here that the value obtained is $\|v\|_2/H_{n-1}.$ Then we will construct the Lagrange dual to the optimization problem and find a dual solution also obtaining value $\|v\|_2/H_{n-1}.$ It then will follow from weak duality that $q$ is optimal. 

To show that $q$ is feasible, we must show that the constraint holds. For this calculation, it is more convenient to work in coordinates in which $v$ is aligned with the first unit vector. To this end, set $u_1=v/\|v\|_2$, and let $U=\begin{bmatrix}u_1 & \ldots & u_n\end{bmatrix}$ be an orthogonal matrix. Let $z=U^\top w$. Then, using Lemma~\ref{lem:rotationalInvariance} on rotational invariance of the sphere, $q$ satisfies the constraint if and only if:
\begin{align*}
\int_{\S^{n-1}}U^\top w q(w)\mu_{n-1}(dw)&=
\frac{\|v\|_2}{H_{n-1}}\int_{\S^{n-1}}z\sign(z_1)\mu_{n-1}(dz)\\
&=\frac{\|v\|_2}{H_{n-1}}\int_{\S^{n-1}}
\begin{bmatrix}
|z_1|\\
z_2 \sign(z_1) \\
\vdots \\
z_n\sign(z_1)
\end{bmatrix}
\mu_{n-1}(dz) \\
&=U^\top v = \begin{bmatrix}
\|v\|_2 \\
0 \\
\vdots \\
0
\end{bmatrix}.
\end{align*}

Thus, it suffices to show that 
\begin{align}
\label{eq:coordinateIntegral}
H_{n-1}&=\int_{\S^{n-1}}|z_1|\mu_{n-1}(dz) = \frac{2\pi^{n-1}}{\Gamma\left(\frac{n+1
}{2}\right)} \\
\label{eq:signOrthogonal}
0&=\int_{\S^{n-1}}z_i\sign(z_1)\mu_{n-1}(dz) \textrm{ for } i=2,\ldots,n.
\end{align}

For $n=1$, only \eqref{eq:coordinateIntegral}  must be shown, and in this case, both sides evaluate to $2$. For $n=2$, both sides of \eqref{eq:coordinateIntegral} evaluate to $4$, and \eqref{eq:signOrthogonal} holds by direct calculation. 

For $n\ge 3$, Lemma~\ref{eq:1dIntegral}, followed by  some manipulations gives:
\begin{align*}
\int_{\S^{n-1}}|z_1|\mu_{n-1}(dz)&=A_{n-2}\int_0^\pi |\cos(\phi_1)|\sin^{n-2}(\phi_1)d\phi \\
&=2A_{n-2}\int_0^{\pi/2} \cos(\phi_1)\sin^{n-2}(\phi_1)d\phi \\
&=\frac{2A_{n-2}}{n-1} \\
&=\frac{4\pi^{\frac{n-1}{2}}}{(n-1)\Gamma\left(\frac{n-1}{2}\right)}\\
&=\frac{2\pi^{\frac{n-1}{2}}}{\Gamma\left(\frac{n+1}{2}\right)}.
\end{align*}
Thus, \eqref{eq:coordinateIntegral} holds for all $n\ge 1$.

Now we evaluate the integrals from \eqref{eq:signOrthogonal} via Lemma~\ref{lem:coordinateIntegral}. 
%
For $i=2,\ldots,n$, integrating over $\phi_2, \ldots, \phi_{n-1}$ gives:
\begin{align*}
\int_{\S^{n-1}}z_i\sign(z_1)\mu_{n-1}(dz)
&\propto \int_{0}^\pi \sign(\cos(\phi_1))\sin^{n-1}(\phi_1)d\phi_1 \\
&=\int_0^{\pi/2}\sin^{n-1}(\phi_1)d\phi_1 - \int_{\pi/2}^{\pi}\sin^{n-1}(\phi_1)d\phi_1 = 0.
\end{align*}
The final equality follows from substitution $\psi=\pi-\phi_1$ in the second integral:
\begin{align*}
\int_{\pi/2}^{\pi}\sin^{n-1}(\phi_1)d\phi_1 &=\int_0^{\pi/2}\sin^{n-1}(\pi-\psi)d\psi \\
&=\int_0^{\pi/2}\sin^{n-1}(\psi)d\psi. 
\end{align*}
Thus, \eqref{eq:signOrthogonal} holds for all $n\ge 1$ and all $2\le i \le n$.

Since \eqref{eq:coordinateIntegral} and \eqref{eq:signOrthogonal} hold, the function $q$ is feasible, giving objective value $\|v\|_2/H_{n-1}.$

Now we will derive the Lagrange dual, and find a dual solution obtaining value $\|v\|_2/H_{n-1}$.

The optimization problem from the lemma statement can be posed equivalently as:
\begin{align*}
&\min_{t,f} && t\\
&\textrm{subject to} && \int_{\S^{n-1}}w f(w)\mu_{n-1}(dw)=v \\
&&& -t\le f(w)\le t \textrm{ for almost all } w\in\S^{n-1}.
\end{align*}
This is an infinite dimensional linear program over variables $(t,f)\in \R\times L^{\infty}(\S^{n-1})$.

The Lagrangian of the reformulated problem is given by
\begin{align*}
L(t,f,\lambda,\alpha,\beta)
&=t+\lambda^\top \left(v- \int_{\S^{n-1}}w f(w)\mu_{n-1}(dw)\right)
\\ &+\int_{\S^{n-1}}\left(-\alpha(w)(t+f(w))+\beta(w)(f(w)-t) \right)\mu_{n-1}(dw) \\
&=\lambda^\top v+
t\left(1-\int_{\S^{n-1}}(\alpha(w)+\beta(w))\mu_{n-1}(dw) \right)\\
&+\int_{\S^{n-1}}f(w)\left(\beta(w)-\alpha(w)-\lambda^\top w \right)\mu_{n-1}(dw),
\end{align*}
with dual variables $(\lambda,\alpha,\beta)\in \R^n\times L^1(\S^{n-1})\times L^1(\S^{n-1}).$

The associated dual problem is given by:
\begin{align*}
&\max_{\lambda,\alpha,\beta} && \lambda^\top v \\
&\textrm{subject to} && \alpha(w)\ge 0, \beta(w)\ge 0, \textrm{ for almost all } w\in \S^{n-1} \\
&&& \int_{\S^{n-1}}(\alpha(w)+\beta(w))\mu_{n-1}(dw)=1\\
&&& \beta(w)-\alpha(w)=\lambda^\top w, \textrm{ for almost all } w\in \S^{n-1}.
\end{align*}

We claim that the dual problem is equivalent to:
\begin{subequations}
\label{eq:vDual}
\begin{align}
&\max_{\lambda,\alpha,\beta} && \lambda^\top v \\
&\textrm{subject to} && \int_{\S^{n-1}}|\lambda^\top w|\mu_{n-1}(dw)\le 1.
\end{align}
\end{subequations}

Indeed, for a dual feasible $(\lambda,\alpha,\beta)$, we must have $|\lambda^\top w|\le \alpha(w)+\beta(w)$, and so $\lambda$ is feasible for \eqref{eq:vDual}.

Conversely, given any $\lambda$ feasible for \eqref{eq:vDual}, let 
$$
s=\int_{\S^{n-1}}|\lambda^\top w|\mu_{n-1}(dw)\le 1.
$$
Then, we can construct corresponding dual feasible $\alpha$ and $\beta$ by setting
\begin{align*}
\alpha(w)&=\begin{cases}
-\lambda^\top w +\frac{1-s}{2A_{n-1}}& \lambda^\top w <0 \\
\frac{1-s}{A_{n-1}} & \lambda^\top w \ge 0
\end{cases}\\
\beta(w)&=\begin{cases}
\lambda^\top w + \frac{1-s}{2A_{n-1}} &\lambda^\top w > 0 \\
\frac{1-s}{2A_{n-1}}& \lambda^\top w \le 0.
\end{cases}
\end{align*}

Recall that we are examining the case that $v\ne 0.$
Let 
$$
\lambda = \frac{1}{\|v\|_2 H_{n-1}}v.
$$
Note that $\lambda^\top v = \|v\|_2/H_{n-1}$, which was the value obtained by $q$ on the primal problem. As discussed above, $\lambda$ will correspond to a dual feasible solution as long as it is feasible for \eqref{eq:vDual}.

Recall the change of coordinates from above, $z=U^\top w$, where $z_1=v^\top w/\|v\|_2$. Using rotational invariance of the sphere, Lemma~\ref{lem:rotationalInvariance},
gives:
\begin{align*}
\int_{\S^{n-1}}|\lambda^\top w|\mu_{n-1}(dw)&=\frac{1}{\|v\|_2 H_{n-1}}\int_{\S^{n-1}}|v^\top w|\mu_{n-1}(dw)\\
&=\frac{1}{H_{n-1}}\int_{\S^{n-1}}|z_1|\mu_{n-1}(dz) \\
&=1,
\end{align*}
where the final equality is from \eqref{eq:coordinateIntegral}.

Thus, $\lambda$ is feasible for \eqref{eq:vDual}. By weak duality, the value achieved, $\lambda^\top v=\|v\|_2/H_{n-1}$,
is a lower bound on the achievable value for the optimization problem from the lemma statement. Thus, $q$ must be  optimal. 
\end{proof}

\begin{corollary}
\label{cor:linRep}
For $n\ge 1$ and $v\in \R^n$, let $q(w)=\frac{\|v\|_2}{H_{n-1}}\sign(v^\top w)$, where $H_{n-1}=\frac{2\pi^{n-1}}{\Gamma\left(\frac{n+1}{2}\right)}$. The function $p(w)=\frac{1}{R}q(w)$ has $\|p\|_{L^{\infty}(\S^{n-1})}=\frac{\|v\|_2}{H_{n-1} R}$ and satisfies
$$
\int_{-R}^R\int_{\S^{n-1}}p(w)\sigma(w^\top x+b)\mu_{n-1}(dw)db=v^\top x
$$
for all $x\in\R^n.$
\end{corollary}

\begin{proof}
The value of $\|p\|_{L^{\infty}(\S^{n-1})}$ is a direct calculation. 

Using Lemma~\ref{lem:optLinear}, followed by the identity $t=\sigma(t)-\sigma(-t)$ gives:
\begin{align*}
v^\top x &= \frac{1}{2R}\int_{-R}^R\int_{\S^{n-1}}q(w)\left(w^\top x+b\right)\mu_{n-1}(dw)db \\
&=\frac{1}{2R}\int_{-R}^R\int_{\S^{n-1}}q(w)\left(\sigma\left(w^\top x+b\right)
-\sigma\left(-w^\top x-b\right)
\right)
\mu_{n-1}(dw)db.
\end{align*}

Then using the change of coordinates $\hat w = -w$ and $\hat b=-b$ (and rotational invariance of $\S^{n-1}$) gives
$$
\int_{-R}^R\int_{\S^{n-1}}q(w)\sigma(-w^\top x-b)\mu_{n-1}(dw)db=\int_{-R}^R\int_{\S^{n-1}}q(-\hat w)\sigma(\hat w^\top x +\hat b)\mu_{n-1}(d\hat w)d\hat b.
$$
Plugging this equality in the expression above gives that:
$$
v^\top x = \frac{1}{2R}\int_{-R}^R\int_{\S^{n-1}}\left(q(w)-q(-w)\right)\sigma(w^\top x + b)\mu_{n-1}(dw)db.
$$
The result follows after noting that $\sign(t)-\sign(-t)=2\sign(t)$ for all $t\in\R$.
\end{proof}

\begin{lemma}
\label{lem:integralNoAffine}
Say that $n\ge 1$ and $R>0$
If $g:\R^n\to\R$ has $\|g\|_{F^{n+3}}<\infty$, then there is a function $\ell:\S^{n-1}\times [-R,R]$ such that 
\begin{equation}
\label{eq:integralNoAffine}
g(x)=\int_{\S^{n-1}}\int_{-R}^R \ell(w,b) \sigma(w^\top x + b)db \mu_{n-1}(dw),
\end{equation}
for almost all $\|x\|_2\le R$. Furthermore,
$$
\|\ell\|_{L^{\infty}(\S^{n-1}\times [-R,R])}\le 
\left(1+2\frac{1+R}{R^2}+\frac{1}{R}\sqrt{\frac{n \pi}{2}}\right)
\frac{2}{(2\pi)^n}
\|g\|_{F^{n+3}}.
$$
\end{lemma}

\begin{proof}
Let $\xi$, $r$, and $v$ be the function, number and vector from Lemma~\ref{lem:integral}. Let $\zeta$ be the function from Corollary~\ref{cor:constRep} corresponding to $r$ and let $p$ be the function from Corollary~\ref{cor:linRep} corresponding to $v$. Then, by construction
$$
\ell(w,b)=\xi(w,b)+\zeta(b)+p(w)
$$
satisfies the integral representation from \eqref{eq:integralNoAffine} for almost all $\|x\|_2\le R$. 

We bound $\|\ell\|_{L^{\infty}(\S^{n-1}\times [-R,R])}$ via the triangle inequality, followed by the bounds on $\|g\|_{L^{\infty}(\S^{n-1}\times [-R,R])}$, $|r|$, and $\|v\|_2$:
\begin{align*}
\|\ell\|_{L^{\infty}(\S^{n-1}\times [-R,R])}&\le \|\xi\|_{L^{\infty}(\S^{n-1}\times[-R,R])}
+\|\zeta\|_{L^{\infty}([-R,R])}+\|p\|_{L^{\infty}(\S^{n-1})} \\
&\le \frac{2}{(2\pi)^n}\|g\|_{F^{n+3}}+\frac{2|r|}{A_{n-1}R^2}+\frac{\|v\|_2}{H_{n-1}R} \\
&\le \frac{2}{(2\pi)^n}\|g\|_{F^{n+3}}+\frac{\frac{4A_{n-1}}{(2\pi)^n}\left( 1+ R\right) \|g\|_{F^{n+3}}}{A_{n-1}R^2}
+\frac{\frac{ 2 A_{n-1}}{(2\pi)^n} \|g\|_{F^{n+3}} }{H_{n-1}R}\\
&=\left(1+2\frac{1+R}{R^2}+\frac{1}{R}\frac{2\pi^{n/2}}{\Gamma(n/2)}\frac{\Gamma\left(\frac{n+1}{2}\right)}{2 \pi^{\frac{n-1}{2}}}\right)
\frac{2}{(2\pi)^n}
\|g\|_{F^{n+3}}
\\
&\le 
\left(1+2\frac{1+ R}{R^2}+\frac{1}{R}\sqrt{\frac{n\pi}{2}}\right)
\frac{2}{(2\pi)^n}
\|g\|_{F^{n+3}}
\end{align*}
The final inequality uses that for $n\ge 2$
$$
\frac{\Gamma\left(\frac{n+1}{2}\right)}{\Gamma\left(\frac{n}{2}\right)}< \sqrt{\frac{n-1}{2}}
<\sqrt{\frac{n}{2}},
$$
by Gautschi's inequality. For $n=1$, direct calculation gives
$$
\frac{\Gamma\left(\frac{n+1}{2}\right)}{\Gamma\left(\frac{n}{2}\right)}=\frac{1}{\sqrt{\pi}}<
\sqrt{\frac{n}{2}}.
$$
\end{proof}

\subsection{Proof of Proposition~\ref{prop:generalApproximation}}
\label{apss:approximationPf}

We now complete the proof of the approximation result. We refine the argument from \cite{lamperski2024functiongradientapproximationrandom}. The differences are as follows:
\begin{itemize}
\item The approximation here is based on the integral representation from Lemma~\ref{lem:integralNoAffine}, which removes the affine terms.
\item Logarithmic dependence  on $m$, the number of neurons, is removed via the Dudley entropy bound from Lemma~\ref{lem:dudley}.
\item  The approximation is proved to hold up to a set of measure zero, since here, we only assume the inverse Fourier transform relation (and thus the integral representation) hold almost everywhere.
\end{itemize}

Let $\bm{w}_i$ be uniform over $\S^{n-1}$ and $\bm{b}_i$ be uniform over $[-R,R]$ with all random variables independent. Note that $(\bm{w}_i,\bm{b}_i)$ has density $\frac{1}{2RA_{n-1}}$ over $\S^{n-1}\times [-R,R]$. Set 
\begin{align*}
    \bm{c}_i &=\frac{2R A_{n-1} \ell(\bm{w}_i,\bm{b}_i)}{m} \\
    \bm{\zeta}_i(x)&=2 R A_{n-1}\ell(\bm{w}_i,\bm{b}_i)\sigma(\bm{w}_i^\top x+\bm{b}_i)-g(x) \\
    \bm{\gamma}(x)&=\frac{1}{m}\sum_{i=1}^m\bm{\zeta}_i(x).
\end{align*}

The bound on $|\bm{c}_i|$ follows from:
\begin{align*}
2R\|\ell\|_{L^{\infty}(\S^{n-1}\times [-R,R])}&\le 2R
\left(1+2\frac{1+ R}{R^2}+\frac{1}{R}\sqrt{\frac{n \pi}{2}}\right)\frac{2}{(2\pi)^n}\|g\|_{F^{n+3}} \\
&\le \left(2R + 4+
3\sqrt{n}  + 4R^{-1}\right)\frac{2}{(2\pi)^n}\|g\|_{F^{n+3}}.
\end{align*}

The rest of the proof focuses on proving that the approximation error, $\bm{\gamma}(x)$, concentrates around $0$.

Equation~\ref{eq:L1Bounds} from the proof of Lemma~\ref{lem:integral} shows that for almost all $x\in \R^n$, the inverse Fourier transform relation holds, \eqref{eq:ift}, and the following bounds hold:
\begin{align*}
|g(x)| &\le \frac{2A_{n-1}}{(2\pi)^{n}}\|g\|_{F^{n+3}}.
\end{align*}
Furthermore, the derivative rule for Fourier transforms gives
$$
\nabla g(x)=\int_{\R^n}e^{j2\pi \omega^\top x}\hat g(\omega)2\pi \omega d\omega, 
$$
for almost all $x$. So, \eqref{eq:L1Bounds} then implies that 
$$
\|\nabla g(x)\|_2\le \frac{2A_{n-1}}{(2\pi)^n}\|g\|_{F^{n+3}}
$$
for almost all $x$.

The bounds above, combined with Lemma~\ref{lem:integral} imply that 
there is a set $\cS\subset B_R$ such that $B_{R}\setminus \cS$ has Lebesgue measure zero, the inverse Fourier transform relation, \eqref{eq:ift} holds on $\cS$, $\bm{\zeta}_i(x)$ have mean zero on $\cS$, and $g$ is bounded by $\frac{2A_{n-1}}{(2\pi)^n}$ on $\cS$, and $g$ is $\frac{2A_{n-1}}{(2\pi)^n}\|g\|_{F^n+3}$-Lipschitz on $\cS$.

If $x\in B_R$, then, $|\bm{w}_i^\top x + \bm{b}_i|\le 2R$. It follows that: 
$$
\|\bm{\zeta}_i\|_{L^{\infty}(B_R)}\le 4R^2 A_{n-1}\|\ell\|_{L^{\infty}(\S^{n-1}\times [-R,R])} + \frac{2A_{n-1}}{(2\pi)^n}\|g\|_{F^{n+3}}:=\beta.
$$
It follows that for all $x\in \cS$, $\bm{\zeta}_i(x)$ is $\beta$-sub-Gaussian.

Let
$$
\bm{z}=\sup_{x\in\cS}|\bm{\gamma}(x)|=\sup_{(x,s)\in\cS\times\{-1,1\}}s\bm{\gamma}(x)=\sup_{(x,s)\in \cS\times \{-1,1\}}\frac{1}{m}\sum_{i=1}^ms\bm{\zeta}_i(x).
$$

The functional Hoeffding theorem (Theorem 3.2.6 of \cite{wainwright2019high}) implies that for all $\epsilon >0$, 
$$
\Prob\left(\bm{z}\ge \E[\bm{z}]+\epsilon\right)\le e^{-\frac{m\epsilon^2}{16\beta^2}}.
$$
Setting the right side equal to $\delta\in (0,1)$ gives
\begin{equation}
\label{eq:functionalHoeffding}
\Prob\left(\bm{z}\ge \E[\bm{z}]+\frac{4\beta}{\sqrt{m}}\sqrt{\log(\delta^{-1})}\right)\le \delta.
\end{equation}

Now, we bound $\E[\bm{z}]$ using Lemma~\ref{lem:dudley}. Since $\bm{\zeta}_i(x)$ are zero mean and $\beta$-sub-Gaussian for all $x\in\cS$, independence implies that $s\bm{\gamma}(x)$ are $(\beta/\sqrt{m})$-sub-Gaussian for all $(x,s)\in\cS\times\{-1,1\}$.

Now we must examine the continuity properties of $\gamma$.
Since $g$ is $\frac{ 2A_{n-1}}{(2\pi)^n}$-Lipschitz and $\sigma$ is $1$-Lipschitz, $\bm{\zeta}_i$ is $L$-Lipschitz, where
$$
L = 2RA_{n-1}\|\ell\|_{L^\infty(\S^{n-1}\times[-R,R])}+\frac{2A_{n-1}}{(2\pi)^n}\|g\|_{F^{n+3}}.
$$
Thus, for all $x,y\in\cS$, $\bm{\zeta}_i(x)-\bm{\zeta}_i(y)$ must then be $L\|x-y\|_2$-sub-Gaussian. It follows that $\bm{\gamma}(x)-\bm{\gamma}(y)$ is $\frac{L\|x-y\|_2}{\sqrt{m}}$-sub-Gaussian. 

For $(x,a),(y,b)\in \cS\times \{-1,1\}$ we bound $\E\left[\exp\left(\lambda \left(a\bm{\gamma}(x)-b\bm{\gamma}(y)\right)\right)\right]$. If $a=b$, then 
\begin{subequations}
\label{eq:gammaSubGauss}
\begin{equation}
\label{eq:gammaSubGaussSame}
\E\left[\exp\left(\lambda \left(a\bm{\gamma}(x)-b\bm{\gamma}(y)\right)\right)\right]\le 
\exp\left(\frac{\lambda^2 L^2\|x-y\|_2^2}{2m}\right).
\end{equation}

If $a\ne b$, then
$$
| a\bm{\zeta}_i(x)-b\bm{\zeta}_i(y)|\le 2\beta. 
$$
It follows that $a\bm{\gamma}(x)-b\bm{\gamma}(y) $ is $\frac{2\beta}{\sqrt{m}}$-sub-Gaussian, in this case. Thus, here we have
\begin{equation}
\label{eq:gammaSubGaussDiff}
\E\left[\exp\left(\lambda \left(a\bm{\gamma}(x)-b\bm{\gamma}(y)\right)\right)\right]\le 
\exp\left(\frac{\lambda^2 4\beta^2}{2m}\right).
\end{equation}
\end{subequations}

Define the metric $d$ over $\cS\times \{-1,1\}$ by
$$
d((x,a),(y,b))=\indic(a=b)\frac{L\|x-y\|_2}{\sqrt{m}}+\indic(a\ne b)\frac{\max\{2\beta,2LR\}}{\sqrt{m}},
$$
where 
$$
\indic(\mathcal{C})=\begin{cases}
1 & \textrm{if condition } \mathcal{C} \textrm{ holds}\\
0 & \textrm{otherwise}.
\end{cases}
$$
The $\max\{2\beta,2LR\}$ term is used instead of just $2\beta$ to ensure that the triangle inequality holds, and so $d$ is a metric. Note that if $d((x,a),(y,b))< 2LR/\sqrt{m}$, then $a=b$ must hold.

The inequalities from \eqref{eq:gammaSubGauss} imply that $
a\bm{\gamma}(x)-b\bm{\gamma}(y)
$ is $d((x,a),(y,b))$-sub-Gaussian. Furthermore, we have that
$$
|a\bm{\gamma}(x)-b\bm{\gamma}(y)|\le \sqrt{m}d((x,a),(y,b)).
$$
In other words, $s\bm{\gamma}(x)$ is $\sqrt{m}$-Lipschitz in $(x,s)$. 

For compact notation, set $\cX=\cS\times \{-1,1\}$. Recall that $N(t,\cX,d)$ denotes the $t$-covering number of $\cX$ with respect to the metric $d$. Lemma~\ref{lem:dudley} now gives that for all $\epsilon >0$
\begin{equation}
\label{eq:dudleyBound1}
\E[\bm{z}]\le \frac{\beta}{\sqrt{m}}\sqrt{2\log(N(\epsilon/2,\cX,d))}
+4\int_{0}^{\epsilon}\sqrt{2\log(N(t,\cX,d))}dt.
\end{equation}

If $\cU$ is an $\rho$-covering of $B_1$ with respect to $\|\cdot\|_2$, then the scaled set 
$$
R\cU=\{Ru|u\in\cU\}
$$
is an $(R\rho)$-covering of $B_R$ with respect to $\|\cdot\|_2$, and thus an $\left(\frac{\rho RL}{\sqrt{m}}\right)$-covering of $B_R$ with respect to $\frac{L}{\sqrt{m}}\|\cdot\|_2$. In particular, if $\rho < 2$, then $(R\cU)\times \{-1,1\}$ is a $\left(\frac{\rho RL}{\sqrt{m}}\right)$-covering of $\cS\times \{-1,1\}=\cX$.

Set $\epsilon=\frac{\rho RL}{\sqrt{m}}$ for some $\rho \in (0,2)$. For $t\in (0,\epsilon]$, set $t=\frac{u RL}{\sqrt{m}}$. The argument above shows that
$$
N(t,\cX,d)=N\left(\frac{u RL}{\sqrt{m}},\cX,d \right)\le 2N(u,B_1,\|\cdot\|_2)\le 2\left(1+\frac{2}{u}\right)^n.
$$
The bound on $N(u,B_1,\|\cdot\|_2)$ was given in Example 5.8 of \cite{wainwright2019high}.

Using the substitutions $\epsilon=\frac{\rho RL}{\sqrt{m}}$ and $t=\frac{u RL}{\sqrt{m}}$ in \eqref{eq:dudleyBound1} gives
\begin{equation}
\label{eq:dudleySub}
\E[\bm{z}]\le \frac{\beta}{\sqrt{m}}\sqrt{2\log\left(2\left(1+\frac{4}{\rho}\right)^n\right)}
+\frac{4RL}{\sqrt{m}}\int_0^{\rho}\sqrt{2\log\left(2\left(1+\frac{2}{u}\right)^n\right)}du.
\end{equation}

In principle, $\rho$ could be tuned to optimize the bound. For simplicity, we set $\rho = 0.1$, which leads to:
\begin{align*}
\sqrt{2\log\left(2\left(1+\frac{4}{0.1}\right)^n\right)}&\le \sqrt{n}\sqrt{2\log(2\cdot 41)}
\le 3\sqrt{n}
\end{align*}
and
\begin{align*}
4 \int_0^{0.1}\sqrt{2\log\left(2\left(1+\frac{2}{u}\right)^n\right)}du &\le 
4\sqrt{n}\int_0^{0.1}\sqrt{2\log\left(2\left(1+\frac{2}{u}\right)\right)}du\\
&\le 1.5\sqrt{n}.
\end{align*}

Thus, we have
\begin{equation*}
\E[\bm{z}]\le \left(3\beta
+\frac{3RL}{2}\right)\sqrt{\frac{n}{m}}.
\end{equation*}

Plugging in the definitions of $\beta$ and $L$, followed by the upper bound on $\|\ell\|_{L^{\infty}(\S^{n-1}\times [-R,R])}$ gives
\begin{align*}
\MoveEqLeft[0]
\E[\bm{z}]
\\
&\le 3 \left(
4R^2 A_{n-1}\|\ell\|_{L^{\infty}(\S^{n-1}\times [-R,R])} + \frac{2A_{n-1}}{(2\pi)^n}\|g\|_{F^{n+3}}
\right)\sqrt{\frac{n}{m}}
\\&
+\frac{3}{2} R\left(
2RA_{n-1}\|\ell\|_{L^\infty(\S^{n-1}\times[-R,R])}+\frac{2A_{n-1}}{(2\pi)^n}\|g\|_{F^{n+3}}
\right)\sqrt{\frac{n}{m}}\\
&=3\left(5R^2 \|\ell\|_{L^{\infty}(\S^{n-1}\times[-R,R])}+\left(1+\frac{R}{2}\right)\frac{2}{(2\pi)^n}\|g\|_{F^{n+3}} \right) A_{n-1}\sqrt{\frac{n}{m}}
\\
&\le 
3\left(5R^2 
\left(1+2\frac{1+R}{R^2}+\frac{1}{R}\sqrt{\frac{n \pi}{2}}\right)
\frac{2}{(2\pi)^n}
\|g\|_{F^{n+3}}
+\left(1+\frac{R}{2}\right)\frac{2}{(2\pi)^n}\|g\|_{F^{n+3}} \right) A_{n-1}\sqrt{\frac{n}{m}}
\\
&\le 
\left(15 R^2+32R +19 \sqrt{n}R+33\right) \frac{2A_{n-1}}{(2\pi)^n}\|g\|_{F^{n+3}}\sqrt{\frac{n}{m}}.
\end{align*}

A similar argument gives that
\begin{align*}
4\beta &\le 4 \left(
4R^2 \|\ell\|_{L^{\infty}(\S^{n-1}\times [-R,R])} + \frac{2}{(2\pi)^n}\|g\|_{F^{n+3}}
\right)A_{n-1} \\
&\le
4 \left(
4R^2
\left(1+2\frac{1+R}{R^2}+\frac{1}{R}\sqrt{\frac{n \pi}{2}}\right)
\frac{2}{(2\pi)^n}
\|g\|_{F^{n+3}}
+ \frac{2}{(2\pi)^n}\|g\|_{F^{n+3}}
\right)A_{n-1} \\
&\le \left(
16R^2+32R+21\sqrt{n}R+36
\right)\frac{2A_{n-1}}{(2\pi)^n}\|g\|_{F^{n+3}}.
\end{align*}

Plugging the bounds on $\E[\bm{z}]$ and $4\beta$ into \eqref{eq:functionalHoeffding} gives the result.
\hfill$\blacksquare$

\section{Proof of the  Main Result}
\label{app:proof}

As we will see, the KL estimate has two sources of error: sub-optimality of the parameters found by the algorithm and approximation error due to using our specific random feature expansion. The sub-optimality is quantified in Subsections~\ref{appss:constPf}, ~\ref{appss:normalization} and \ref{appss:cvx}. The approximation error is quantified in Subsection~\ref{appss:klApprox}. These bounds are combined to complete the proof in Subsection~\ref{appss:pf}.

\subsection{Quantities for Optimization Error}
\label{appss:constPf}

The optimization error depends on a variety of quantities based on the geometry of  the domain, the smoothness of the functions, and the variance of the estimates of $z^\star$. These quantities are collected in the lemma below. 

\begin{lemma}
\label{lem:constants}
Let Assumptions~\ref{as:support} and \ref{as:smooth} hold.
Define $\Theta$ by \eqref{eq:constraints}.
\begin{enumerate}
\item  \label{it:DTheta}
If $\theta\in \Theta$, then $\|\theta\|_2\le C_{\Theta}/\sqrt{m}$. Furthermore, the diameter of $\Theta$ is $D_{\Theta}:=2C_{\Theta}/\sqrt{m}.$
\item \label{it:DZ}
Let $\cZ=[e^{-2RC_{\Theta}},e^{2RC_{\Theta}}]$. Then:
\begin{enumerate}
\item  $\cZ$ has  diameter  $e^{2RC_{\Theta}}-e^{-2RC_{\Theta}}\le e^{2RC_{\Theta}}:=D_{\cZ}.$
\item  If $\bm{z}_0\in \cZ$, then $\bm{z}_k\in \cZ$ for all $k\ge 0$.
\end{enumerate}
\item \label{it:zStarLipschitz}
Let $\bm{z}^\star(\theta)=\E_{\Q}[e^{\bm{\phi}(\bm{y})^\top \theta}]$. The function $\bm{z}^\star$ is $L_z$-Lipschitz, where $L_z=2R\sqrt{m}e^{2RC_{\Theta}}.$
\item \label{it:FLipschitz}
For $\zeta=(x,y)$, let $\bm{F}(\theta,z,\zeta)=\bm{\phi}(x)-\frac{1}{z}e^{\bm{\phi}(y)^\top \theta}\bm{\phi}(y)$. 
\begin{enumerate}
\item   For $(\theta,z,\zeta)\in \Theta\times\cZ\times \Omega^2$,  $\|\bm{F}(\theta,z,\zeta)\|_2\le 2R\sqrt{m}(1+e^{4RC_{\Theta}})=:G$
    \item For fixed $(\theta,\zeta)\in \Theta\times \Omega^2$ function $\bm{F}(\theta,\cdot,\zeta)$ is $L_F$-Lipschitz with respect to $z$, where $L_F=2R\sqrt{m}e^{6RC_{\Theta}}$.
\end{enumerate}
\item  \label{it:variance} 
$\E_{\Q}\left[\left(e^{\bm{\phi}(\bm{y})^\top \theta} - \bm{z}^\star(\theta)\right)^2 \right]\le e^{4R C_{\Theta}}=:\nu^2$
\end{enumerate}
\end{lemma}

\begin{proof}

(\ref{it:DTheta}): If $\theta\in\Theta$, then 
\begin{align*}
\|\theta\|_2^2&=\sum_{i=1}^m\theta_i^2 \\
&\le m\frac{C_{\Theta}^2}{m^2} = \frac{C_{\Theta}^2}{m},
\end{align*}
with equality achieved by choosing $\theta_i=C_{\Theta}$ for all $i$. The diameter calculation is similar. 

(\ref{it:DZ}): 
For $y\in \Omega\subset B_R$, we have
\begin{equation}
\label{eq:featureSize}
\|\bm{\phi}(y)\|_2^2 = \sum_{i=1}^m \sigma(\bm{w}_i^\top x + \bm{b}_i)^2 \le m (2R)^2.
\end{equation}
Thus, if $\theta\in\Theta$, the Cauchy-Schwarz inequality gives:
\begin{equation}
\label{eq:nnSize}
|\bm{\phi}(y)^\top \theta|\le 2R \sqrt{m} C_{\Theta}/\sqrt{m}= 2R C_{\Theta}.
\end{equation}

Set $\cZ=[e^{-2RC_{\Theta}},e^{2RC_{\Theta}}]$. The diameter calculation for $\cZ$ is immediate. Furthermore, $e^{\bm{\phi}(\bm{y}_{k})^\top \bm{\theta}_k}\in\cZ$ for all $k\ge 0.$

Note that the update rule for $\bm{z}_k$ can be expressed as:
$$
\bm{z}_{k+1}=(1-\alpha)\bm{z}_k+\alpha e^{\bm{\phi}(\bm{y}_k)^\top \bm{\theta}_k},
$$
so that $\bm{z}_{k+1}$ is a convex combination of $\bm{z}_k$ and $e^{\bm{\phi}(\bm{y}_k)^\top \bm{\theta}_k}$. Thus, if $\bm{z}_k$ and $e^{\bm{\phi}(\bm{y}_k)^\top \bm{\theta}_k}$ are both in $\cZ$, we must have that $\bm{z}_{k+1}\in \cZ$. 

(\ref{it:zStarLipschitz}): $z^\star$ is differentiable, with
$$
\nabla \bm{z}^\star(\theta)=\E_{\Q}\left[\bm{\phi}(\bm{y})e^{\bm{\phi}(\bm{y})^\top\theta}\right].
$$
Then using \eqref{eq:featureSize} and \eqref{eq:nnSize} gives 
$$
\|\nabla \bm{z}^\star(\theta)\|_{2}\le 2R\sqrt{m}e^{2R C_{\Theta}}
$$

(\ref{it:FLipschitz}):
\begin{align*}
\|\bm{F}(\theta),z,\zeta)\|_2 &=\left\|
\bm{\phi}(x)-\frac{e^{\bm{\phi}(y)^{\top}\theta}}{z}\bm{\phi}(y)
\right\|_2 \\
&\le \|\bm{\phi}(x)\|_2 + \frac{e^{\bm{\phi}(y)^{\top}\theta}}{z}\|\bm{\phi}(y)\|_2 \\
&\le 2R\sqrt{m}+e^{2RC_{\Theta}}\cdot e^{2RC_{\Theta}} 2R\sqrt{m}\\
&=2R\sqrt{m}\left(1+e^{4RC_{\Theta}}\right).
\end{align*}
\begin{align*}
\|\bm{F}(\theta,z_1,\zeta)-\bm{F}(\theta,z_2,\zeta)\|_2 &=
\left|
\frac{1}{z_1}-\frac{1}{z_2}
\right| \cdot \left\|\bm{\phi}(y)e^{\bm{\phi}(y)^\top \theta} \right\|_2 \\
&\le 2R \sqrt{m}e^{2RC_{\Theta}} \frac{|z_1-z_2|}{z_1 z_2} \\
&\le 2R \sqrt{m}e^{6RC_{\Theta}}.
\end{align*}

(\ref{it:variance}):
Since $\bm{z}^\star(\theta)=\E_{\Q}[e^{\bm{\phi}(\bm{y})^\top\theta}]$ gives the minimum mean-squared error estimate of $e^{\bm{\phi}(\bm{y})^\top \theta}$, conditioned on $(\bm{w},\bm{b})$, we have:
\begin{align*}
\E_{\Q}\left[\left(e^{\bm{\phi}(\bm{y})^\top\theta} -z^\star(\theta)\right)^2\right] &\le 
\E_{\Q}\left[\left(e^{\bm{\phi}(\bm{y})^\top\theta}\right)^2\right] \\
&\le e^{4RC_{\Theta}}  .
\end{align*}
\end{proof}

\subsection{Normalization Constant Estimation Error}
\label{appss:normalization}
For compact notation, iterative updates from \eqref{eq:algorithm} can be expressed as:
\begin{align*}
    \mb{z}_{k+1} &= \mb{z}_k + \alpha(\bm{g}(\bm{\theta}_k, \bm{\zeta}_k) - \mb{z}_k), \\
    \bm{\theta}_{k+1} &= \Pi_{\Theta}\left(\bm{\theta}_k + \alpha r \bm{F}(\bm{\theta}_k, \mb{z}_k, \bm{\zeta}_k)\right),
\end{align*}
where
\begin{align*}
\bm{g}(\theta,\zeta)&=e^{\bm{\phi}(y)^\top\theta} \\
\bm{F}(\theta,z,\zeta)&=\bm{\phi}(x)-\frac{1}{z}e^{\bm{\phi}(y)^\top \theta} \bm{\phi}(y).
\end{align*}

Recall from Lemma~\ref{lem:constants} that $\E[(\bm{g}(\theta,\bm{\zeta})-\bm{z}^\star(\theta))^2|\bm{w},\bm{b}]\le \nu^2$ for all $\theta\in\Theta$.

Let $G\ge \sup_{(\theta,z,\zeta)\in (\Theta,\cZ,\Omega^2)}\bm{F}(\theta,z,\zeta)$ be the upper bound from Lemma~\ref{lem:constants}, where $\cZ$ and $\bm{z}^\star$ were defined in Lemma~\ref{lem:constants}, and $\Theta$ was defined by \eqref{eq:constraints}.

Also recall from Lemma~\ref{lem:constants} that $\bm{z}^\star$ is $L_z$-Lipschitz and $\cZ$ has diameter less than $D_{\cZ}$.

\begin{lemma}
\label{lem:zError}
If $\alpha< 1$, then for all $k\ge 0$
$$
\E[|\bm{z}_k-\bm{z}^\star(\bm{\theta}_k)||\bm{w},\bm{b}]\le 
(1-\alpha)^{k}D_{\cZ} + \alpha r L_z G + \sqrt{\alpha} \nu. 
$$
\end{lemma}

\begin{proof}
For all $k\ge 0$:
\begin{align*}
\bm{z}_{k+1}-\bm{z}^\star(\bm{\theta}_{k+1})&=
\bm{z}_{k+1}-\bm{z}^\star(\bm{\theta}_k)+\bm{z}^\star(\bm{\theta}_k)-\bm{z}^\star(\bm{\theta}_{k+1}) \\
&=\left(\bm{z}_k+\alpha (g(\bm{\theta}_k,\bm{\zeta}_k)-\bm{z}_k)-\bm{z}^\star(\bm{\theta}_k) \right)+
\bm{z}^\star(\bm{\theta}_k)-\bm{z}^\star(\bm{\theta}_{k+1})\\
&=(1-\alpha)(\bm{z}_k-\bm{z}^\star(\bm{\theta}_k))+
\alpha (g(\bm{\theta}_k,\bm{\zeta}_k)-\bm{z}^\star(\bm{\theta}_k))+\bm{z}^\star(\bm{\theta}_k)-\bm{z}^\star(\bm{\theta}_{k+1}).
\end{align*}

Iterating this equality gives that
\begin{multline*}
\bm{z}_{k}-\bm{z}^\star(\bm{\theta}_k)=(1-\alpha)^{k}(\bm{z}_0-\bm{z}^\star(\bm{\theta}_0))
+\alpha \sum_{i=0}^{k-1}(1-\alpha)^{k-1-i}\left(
\left(
g(\bm{\theta}_i,\bm{\zeta}_i)-\bm{z}^\star(\bm{\theta}_i)\right)
+\left(
\bm{z}^\star(\bm{\theta}_i)-\bm{z}^\star(\bm{\theta}_{i+1})
\right)
\right).
\end{multline*}

Thus.
\begin{multline}
\label{eq:zErrorAbs}
|\bm{z}_{k}-\bm{z}^\star(\bm{\theta}_k)|=(1-\alpha)^{k}|\bm{z}_0-\bm{z}^\star(\bm{\theta}_0)|\\
+\left|\alpha \sum_{i=0}^{k-1}(1-\alpha)^{k-1-i}\left(
\bm{g}(\bm{\theta}_i,\bm{\zeta}_i)-\bm{z}^\star(\bm{\theta}_i)\right)
\right|
+\alpha \sum_{i=0}^{k-1}
(1-\alpha)^{k-1-i}|
\bm{z}^\star(\bm{\theta}_i)-\bm{z}^\star(\bm{\theta}_{i+1})|.
\end{multline}

We bound the terms on the right individually. Using that $\cZ$ is bounded with diameter less than $D_{\cZ}$ gives
$$
(1-\alpha)^{k}|\bm{z}_0-\bm{z}^\star(\bm{\theta}_0)|\le 
(1-\alpha)^{k}D_{\cZ}.
$$

For the third term, we use that $\bm{z}^\star$ is $L_z$-Lipschitz and that  $\bm{F}$ is bounded to give
\begin{align*}
\alpha \sum_{i=0}^{k-1}
(1-\alpha)^{k-1-i}|
\bm{z}^\star(\bm{\theta}_i)-\bm{z}^\star(\bm{\theta}_{i+1})|&\le 
\alpha^2 r L_z G \sum_{i=0}^{k-1}
(1-\alpha)^{k-1-i}\\
&\le \alpha r L_z G.
\end{align*}

The second term on the right of \eqref{eq:zErrorAbs}, we only bound in expectation.

Using that $\bm{\zeta}_i$ are IID, with $\E\bm{g}[(\bm{\theta}_i,\bm{\zeta}_i)|\bm{\theta}_i,\bm{w},\bm{b}]=\bm{z}^\star(\bm{\theta}_i)$ gives
\begin{align*}
\MoveEqLeft
\E\left[\left|\alpha \sum_{i=0}^{k-1}(1-\alpha)^{k-1-i}\left(
g(\bm{\theta}_i,\bm{\zeta}_i)-z^\star(\bm{\theta}_i)\right)
\right|
\middle|\bm{w},\bm{b}
\right]\\
&\le 
\sqrt{
\E\left[\left|\alpha \sum_{i=0}^{k-1}(1-\alpha)^{k-1-i}\left(
g(\bm{\theta}_i,\bm{\zeta}_i)-z^\star(\bm{\theta}_i)\right)
\right|^2\middle|\bm{w},\bm{b}\right]
} \\
&=\sqrt{
\E\left[
\alpha^2 \sum_{i=0}^{k-1}(1-\alpha)^{2(k-1-i)}
\left(
g(\bm{\theta}_i,\bm{\zeta}_i)-z^\star(\bm{\theta}_i)\right)^2
\middle|\bm{w},\bm{b}\right]
}
\\
&\le\sqrt{
\alpha^2 \nu^2 \sum_{i=0}^{k-1}(1-\alpha)^{2(k-1-i)}
}\\
&\le \sqrt{
\frac{
\alpha^2\nu^2
}{2\alpha-\alpha^2}
}\le \sqrt{\alpha}\nu.
\end{align*}
In the final inequality, we used that $\alpha \le 1$.

The result follows by plugging the various bounds into \eqref{eq:zErrorAbs}.
\end{proof}

\subsection{Convex Optimization Analysis}
\label{appss:cvx}

Recall the constants $D_{\Theta}$, $D_{\cZ}$, $L_F$, $L_z$, $G$, and $\nu$ from Lemma~\ref{lem:constants}.

\begin{lemma}
\label{lem:cvx}
Let $\overline{\bm{\theta}}_T=\frac{1}{T}\sum_{k=0}^{T-1}\bm{\theta}_k$. For all choices of the weights and biases $(\bm{w},\bm{b})$, all $T\ge 1$ and all choices of $\alpha \in (0,1)$ and $r>0$, we have:
\begin{equation*}
\E[\bm{f}(\overline{\bm{\theta}}_T)|\bm{w},\bm{b}]-\min_{\theta\in\Theta}\bm{f}(\theta)\le 
 \frac{L_F D_{\Theta} D_{\cZ}}{\alpha T}+\frac{D_{\Theta}^2}{2\alpha r T}+
\alpha r\left(
L_FL_z D_{\Theta}G+\frac{G^2}{2}
\right)
+\sqrt{\alpha}\nu L_F D_{\Theta}.
\end{equation*}
\end{lemma}
\begin{proof}

Let $\bm{\theta}^\star$ be a minimizer of $\bm{f}$ over $\Theta$. (Note that $\bm{\theta}^\star$ is a random variable, since the objective function, $\bm{f}$, depends on the random neural network weights and biases.)

Using convexity of $\bm{f}$ twice gives
\begin{align*}
\bm{f}(\overline{\bm{\theta}}_T)&\le \frac{1}{T}\sum_{k=0}^{T-1}\bm{f}(\bm{\theta}_k) \\
&\le \bm{f}(\bm{\theta}^\star)+\frac{1}{T}\sum_{k=0}^{T-1}
\nabla \bm{f}(\bm{\theta}_k)^\top (\bm{\theta}_k-\bm{\theta}^\star).
\end{align*}

So, it suffices to bound:
\begin{multline}
\label{eq:convexSumExpand}
\sum_{k=0}^{T-1}
\nabla \bm{f}(\bm{\theta}_k)^\top (\bm{\theta}_k-\bm{\theta}^\star)
=-\sum_{k=0}^{T-1}\bm{F}(\bm{\theta}_k,\bm{z}_k,\bm{\zeta}_k)^\top (\bm{\theta}_k-\bm{\theta}^\star)\\
+\sum_{k=0}^{T-1}\left(\bm{F}(\bm{\theta}_k,\bm{z}_k,\bm{\zeta}_k)-
\bm{F}(\bm{\theta}_k,\bm{z}^\star(\bm{\theta}_k),\bm{\zeta}_k)
\right)^\top (\bm{\theta}_k-\bm{\theta}^\star)
\\
+\sum_{k=0}^{T-1}
\left(\bm{F}(\bm{\theta}_k,\bm{z}^\star(\bm{\theta}_k),\bm{\zeta}_k)+
\nabla \bm{f}(\bm{\theta}_k)
\right)^\top (\bm{\theta}_k-\bm{\theta}^\star).
\end{multline}

The third term on the right vanishes in expectation.
We bound the first two terms on the right of \eqref{eq:convexSumExpand} individually.

For the first term, using non-expansiveness of convex projections gives
\begin{align*}
\|\bm{\theta}_{k+1}-\bm{\theta}^\star\|_2^2
&\le \|\bm{\theta}_k-\bm{\theta}^\star\|_2^2+2\alpha r \bm{F}(\bm{\theta}_k,\bm{z}_k,\bm{\zeta}_k)^\top (\bm{\theta}_k-\bm{\theta}^\star) 
+(\alpha r)^2 G^2.
\end{align*}
Here, we used the algorithm definition and the bound on $\bm{F}$.

Recall that $\Theta$ has diameter $D_{\Theta}$. 
Thus,
\begin{align*}
-\sum_{k=0}^{T-1}F(\bm{\theta}_k,\bm{z}_k,\bm{\zeta}_k)^\top (\bm{\theta}_k-\bm{\theta}^\star) &\le \sum_{k=0}^{T-1} \left(\frac{1}{2\alpha r}\left(
\|\bm{\theta}_k-\bm{\theta}^\star\|_2^2
-\|\bm{\theta}_{k+1}-\bm{\theta}^\star\|_2^2
\right) \right)+ \frac{\alpha rT G^2}{2} \\
&\le\frac{D_{\Theta}^2}{2\alpha r}+\frac{\alpha r TG^2}{2}.
\end{align*}

For the second term, we use that $F$ is $L_F$-Lipschitz in $z$ to give:
\begin{align*}
\left|\sum_{k=0}^{T-1}\left(F(\bm{\theta}_k,\bm{z}_k,\bm{\zeta}_k)-
F(\bm{\theta}_k,\bm{z}^\star(\bm{\theta}_k),\bm{\zeta}_k)
\right)^\top (\bm{\theta}_k-\bm{\theta}^\star)
\right|&\le L_F D_{\Theta}\sum_{k=0}^{T-1}|\bm{z}_k-\bm{z}^\star(\bm{\theta}_k)|
\end{align*}

Taking expectations and using Lemma~\ref{lem:zError} gives
\begin{align*}
\E\left[
L_F D_{\Theta}\sum_{k=0}^{T-1}|\bm{z}_k-\bm{z}^\star(\bm{\theta}_k)|
\middle|\bm{w},\bm{b}
\right] &\le 
L_F D_{\Theta}\left(
\alpha r T L_z G + \sqrt{\alpha}\sigma T
+
D_{\cZ}
\sum_{k=0}^{T-1}(1-\alpha)^k
\right) \\
&\le 
L_F D_{\Theta}\left(
\alpha r T L_z G + \sqrt{\alpha}\nu T
+
\frac{D_{\cZ}}{\alpha}
\right).
\end{align*}

Plugging the bounds into \eqref{eq:convexSumExpand} gives
\begin{multline*}
\E\left[
\sum_{k=0}^{T-1}\nabla f(\bm{\theta}_k)^\top(\bm{\theta}_k-\bm{\theta}^\star)
\middle| \bm{w},\bm{b}
\right]\\\le \frac{L_F D_{\Theta} D_{\cZ}}{\alpha}+\frac{D_{\Theta}^2}{2\alpha r}+
\alpha r T\left(
L_FL_z D_{\Theta}G+\frac{G^2}{2}
\right)
+\sqrt{\alpha}T \nu L_F D_{\Theta}.
\end{multline*}
Dividing by $T$ now gives the result.
\end{proof}

\subsection{The Optimal KL Approximation}
\label{appss:klApprox}

Recall that $\kappa$ was the approximation error bound from \eqref{eq:estErrorBound}. Recall that $\bm{f}$ defined in \eqref{eq:objective} is a random function depending on the randomly generated weights and biases, $(\bm{w},\bm{b})$.

\begin{lemma}
\label{lem:approx}
If Assumptions~\ref{as:support} and \ref{as:smooth} hold, then for any $\delta \in (0,1)$ with probability (over the weights and biases) at least $1-
\delta$, the following bound holds:
$$
0\le \left(\min_{\theta\in\Theta}\bm{f}(\theta)\right)+D_{KL}(\Prob||\Q)\le 
\frac{2\kappa}{\sqrt{m}}\left(\sqrt{n}+\sqrt{\log(\delta^{-1})}\right).
$$
\end{lemma}
\begin{proof}
The lower bound follows from the Donsker-Varadhan variational characterization:
\begin{align*}
D_{KL}(\Prob||\Q)&= \sup_{T: \Omega \to \R} \left(\E[T(\bm{x})] - \log(\E[e^{T(\bm{y})}]))\right)\\
&\ge \max_{\theta\in\Theta}\left(\E[\bm{\phi}(\bm{x})^\top \theta|\bm{w},\bm{b}]-\log\left(
\E\left[e^{\bm{\phi}(\bm{y})^\top \theta}\middle| \bm{w},\bm{b}\right]
\right)
\right)\\
&=-\min_{\theta\in\Theta}\bm{f}(\theta).
\end{align*}

The upper bound requires a bit more work. By Assumption~\ref{as:smooth}, there is a function $g:\R^n\to \R$ and a constant $\xi$ such that $g(x)=\log\left(\frac{d\Prob}{d\Q}(x)\right)+\xi$ for all $x\in\Omega$ and $\|g\|_{F^{n+3}}\le \rho$. Furthermore, Assumption~\ref{as:support} implies that Proposition~\ref{prop:generalApproximation} can be used to bound the approximation error of $g(x)$ using a random feature expansion from \eqref{eq:randomFeature}. Namely, there must be a parameter vector $\tilde{\bm{\theta}}\in \Theta$ such that for all $\delta \in (0,1) $
$$
\|\bm{\phi}(\cdot)^\top\tilde{\bm{\theta}}-g \|_{L^{\infty}(B_R)}\le \frac{\kappa}{\sqrt{m}}\left(\sqrt{n}+\sqrt{\log(\delta^{-1})}\right)=:\epsilon.
$$

Then
\begin{align*}
\min_{\theta\in\Theta}\bm{f}(\theta)&\le \bm{f}(\tilde{\bm{\theta}}) \\
&=-\E[\bm{\phi}(\bm{x})^\top \tilde{\bm{\theta}}|\bm{w},\bm{b}]+\log\left(
\E\left[e^{\bm{\phi}(\bm{y})^\top \tilde{\bm{\theta}}}\middle| \bm{w},\bm{b}\right]\right) \\
&\le -\E[g(\bm{x})]+\epsilon+\log\left(\E\left[e^{g(\bm{y})+\epsilon}\right]\right) \\
&=-
D_{KL}(\Prob||\Q)+2\epsilon. 
\end{align*}
\end{proof}

\subsection{Proof of  Theorem~\ref{thm:main}}
\label{appss:pf}
Let $\epsilon_{\mathrm{opt}}$ be the optimization error from Lemma~\ref{lem:cvx} and let $\epsilon_{\mathrm{approx}}$ be the approximation error from Lemma~\ref{lem:approx}:
\begin{align*}
    \epsilon_{\mathrm{opt}}&=\frac{L_F D_{\Theta} D_{\cZ}}{\alpha T}+\frac{D_{\Theta}^2}{2\alpha r T}+
\alpha r\left(
L_FL_z D_{\Theta}G+\frac{G^2}{2}
\right)
+\sqrt{\alpha}\nu L_F D_{\Theta}\\
\epsilon_{\mathrm{approx}}&=\frac{2\kappa}{\sqrt{m}}\left(\sqrt{n}+\sqrt{\log(\delta^{-1})}\right).
\end{align*}

Let $\bm{\theta}^\star$ be a minimizer of $\bm{f}$ over $\Theta$. 
Using the Donsker-Vardhan variational characterization, followed by Lemmas~\ref{lem:cvx} and \ref{lem:approx} gives, with probability at least $1-\delta$, with respect to the choice of $\bm{w}$ and $\bm{b}$:
\begin{align*}
0&\le \E[\bm{f}(\overline{\bm{\theta}}_T)|\bm{w},\bm{b}]+D_{KL}(\Prob||\Q) \\
&\le \left(
\E[\bm{f}(\overline{\bm{\theta}}_T)|\bm{w},\bm{b}]-\bm{f}(\bm{\theta}^\star)
\right) 
+\left(
\bm{f}(\bm{\theta}^\star)+
D_{KL}(\Prob||\Q)
\right)\\
&\le \epsilon_{\mathrm{opt}}+\epsilon_{\mathrm{approx}}.
\end{align*}
The first statement now follows by plugging in the definitions of $\epsilon_{\mathrm{opt}}$ and $\epsilon_{\mathrm{approx}}$, and then further separating the dependence of $\epsilon_{\mathrm{opt}}$ on $m$ via the expressions from Lemma~\ref{lem:constants}.

For the second statement, we optimize the parameters defining $\epsilon_{\mathrm{opt}}$. 
In particular, we can write $\epsilon_{\mathrm{opt}}$  in the form:
\begin{equation*}
\epsilon_{\mathrm{opt}}=a_1 \alpha^{-1}+a_2 (\alpha r)^{-1}+a_3 (\alpha r)+a_4 \alpha^{1/2}.
\end{equation*}

In terms of the quantities from the lemma statement:
\begin{align}
\label{eq:constRelation}
a_1 = \frac{b_1}{T}, \quad a_2 = \frac{b_2}{Tm}, \quad a_3 = b_3 m, \quad a_4 = b_4.
\end{align}

Optimizing first over $r>0$ gives 
$$
r=\alpha^{-1}\sqrt{\frac{a_2}{a_3}},
$$
leading to 
$$
\epsilon_{\mathrm{opt}}=a_1 \alpha^{-1}+2\sqrt{a_2 a_3}+a_4 \alpha^{1/2}.
$$

Now, optimizing over $\alpha$ gives:
\begin{equation*}
\alpha = \left(\frac{2 a_1}{a_4}\right)^{2/3},
\end{equation*}
leading to 
$$
\epsilon_{\mathrm{opt}}=\left(2^{-2/3}+2^{1/3}\right)a_1^{1/3}a_4^{2/3}+2\sqrt{a_2a_3}.
$$

So, to compute more explicit values of $\alpha$, $r$, and $\epsilon_{\mathrm{opt}}$,
we  plug in various definitions given in Lemma~\ref{lem:constants} and Equation~\ref{eq:constRelation}:
\begin{align*}
\alpha &=\left(\frac{2\left(\frac{L_FD_{\Theta}D_{\cZ}}{T}\right)}{\nu L_F D_{\Theta}} \right)^{2/3} \\
&=\left(\frac{2D_{\cZ}}{\nu T}\right)^{2/3}\\
&=
\left(\frac{2e^{2RC_{\Theta}}}{Te^{2RC_{\Theta}}}\right)^{2/3}\\
&=\left(\frac{2}{T}\right)^{2/3},
\end{align*}
\begin{align*}
r&=\left(\frac{T}{2}\right)^{2/3}
\left(\frac{b_2}{b_3 Tm^2}\right)^{1/2}\\
&= \frac{T^{1/6}}{m } 2^{-2/3} \sqrt{\frac{b_2}{b_3}},
\end{align*}
and
\begin{align*}
\epsilon_{\mathrm{opt}}&=
\left(2^{-2/3}+2^{1/3}\right)\left(\frac{b_1}{T}\right)^{1/3}\left(b_4\right)^{2/3}+2\sqrt{\frac{b_2 b_3}{T}}.
\end{align*}

\hfill$\blacksquare$

\end{document}